\documentclass[conference]{IEEEtran}

\usepackage{times}

\usepackage[numbers]{natbib}
\usepackage{multicol}
\usepackage{multirow}
\usepackage[bookmarks=true, hidelinks]{hyperref}

\usepackage{siunitx}

\usepackage{amsmath}
\usepackage{amsfonts}
\usepackage{mathtools}
\usepackage{amssymb}
\usepackage{bm}
\usepackage{dsfont}
\usepackage{nicefrac}
\usepackage{accents}
\usepackage{amsthm}
\usepackage{algorithm}
\usepackage{algpseudocode}
\usepackage{xspace}
\usepackage{xcolor}
\usepackage{graphicx}
\usepackage{balance}
\usepackage{booktabs}
\usepackage{wrapfig}
\usepackage{subfigure}
\usepackage[capitalise]{cleveref}
\usepackage{enumerate}
\usepackage{paralist} 
\usepackage{makecell} 
\usepackage{dblfloatfix}

\usepackage{tikz} 
\newtheorem{theorem}{Theorem}
\newtheorem{definition}{Definition}

\setcellgapes{5pt}

\newtheorem{lemma}{Lemma}

\newcommand{\etal}{\emph{et~al.}\xspace}
\newcommand{\setal}{~\emph{et~al.}\xspace}
\newcommand{\eg}{\emph{e.g.,}\xspace}
\newcommand{\ie}{\emph{i.e.,}\xspace}
\newcommand{\myParagraph}[1]{{\bf #1.}\xspace}

\renewcommand{\boldsymbol}[1]{{\bm{#1}}}

\newcommand{\Natural}[1]{ { {\mathbb N}^{#1} } }

\newcommand{\Reals}[1]{ { {\mathbb R}^{#1} } }

\newcommand{\inv}{^{-1}}

\DeclareMathOperator{\Range}{Range}

\newcommand{\at}[1]{^{(#1)}}

\DeclareMathOperator{\ind}{\mathds{1}} 

\newcommand{\vf}{\boldsymbol{f}}

\newcommand{\vv}{\boldsymbol{v}}

\newcommand{\vxx}{\boldsymbol{x}}

\newcommand{\vzz}{\boldsymbol{z}}

\newcommand{\ltrue}{\mathrm{TRUE}} 
\newcommand{\lfalse}{\mathrm{FALSE}} 

\newcommand{\ubar}[1]{\underaccent{\bar}{#1}}

\newcommand\blfootnote[1]{%
  \begingroup
  \renewcommand\thefootnote{}\footnote{#1}%
  \addtocounter{footnote}{-1}%
  \endgroup
}

\newenvironment{edited}{}{}
\newcommand{\edit}[1]{{#1}}

\newcommand{\plausible}{plausible\xspace}
\newcommand{\Plausible}{Plausible\xspace}
\newcommand{\perceived}{perceived\xspace}

\newcommand{\hypgen}{plausible scene generator\xspace}

\newcommand{\HypGen}{Plausible Scene Generator\xspace}
\newcommand{\hypgeneration}{\plausible scene generation\xspace}
\newcommand{\HypGeneration}{\Plausible Scene Generation\xspace}

\newcommand{\statespace}{\mathbb{X}}
\newcommand{\statevar}{\vxx}

\newcommand{\faults}{\vf} 

\newcommand{\hypdistr}{\zeta}
\newcommand{\riskfcn}{\mathcal{R}}
\newcommand{\copula}{C}
\newcommand{\costth}{\theta}
\newcommand{\riskth}{\gamma}

\newcommand{\TaskAware}{Task-Aware\xspace}
\newcommand{\Taskaware}{Task-aware\xspace}
\newcommand{\taskaware}{task-aware\xspace}

\newcommand{\Taskrelevant}{Task-relevant\xspace}
\newcommand{\taskrelevant}{task-relevant\xspace}

\newcommand{\lowbound}[1]{\ubar{#1}}
\newcommand{\upbound}[1]{\bar{#1}}

\newcommand{\datasetURL}{https://anonymous.4open.science/w/task-relevant-risk-estimation/}

\begin{document}

\thanks{Manuscript received January 20, 2002; revised August 13, 2002. }

\title{ \TaskAware Risk Estimation of \\ Perception Failures for Autonomous Vehicles}




%
\author{
    \authorblockN{
        Pasquale Antonante\authorrefmark{1},
        Sushant Veer\authorrefmark{2},
        Karen Leung\authorrefmark{2}\authorrefmark{3}, 
        Xinshuo Weng\authorrefmark{2},
        Luca Carlone\authorrefmark{1}, and
        Marco Pavone\authorrefmark{2}\authorrefmark{4}
    }
    \authorblockA{
        \authorrefmark{1}
        Massachusetts Institute of Technology 
    }
    \authorblockA{
        \authorrefmark{2}
        NVIDIA Research 
    }
    \authorblockA{
        \authorrefmark{3}
        University of Washington
    }
    \authorblockA{
        \authorrefmark{4}
        Stanford University
    }
}

\maketitle
\blfootnote{This work was done while Pasquale was an intern at NVIDIA Research.}

\begin{tikzpicture}[overlay, remember picture]
  \path (current page.north east) ++(-3.75,-0.2) node[below left] {
    This paper has been published in the Proceedings of Robotics: Science and Systems (RSS) 2023.
  };
\end{tikzpicture}
\begin{tikzpicture}[overlay, remember picture]
  \path (current page.north east) ++(-3.85,-0.6) node[below left] {
    Please cite the paper as: P. Antonante, S. Veer, K. Leung, X. Weng, L. Carlone and M. Pavone,
  };
\end{tikzpicture}
\begin{tikzpicture}[overlay, remember picture]
  \path (current page.north east) ++(-0.5,-1) node[below left] {
    ``Task-Aware Risk Estimation of Perception Failures for Autonomous Vehicles'', \emph{Proceedings of Robotics: Science and Systems (RSS)}, 2023.
  };
\end{tikzpicture}
\vspace{-2.5em}

\begin{abstract}
Safety and performance are key enablers for autonomous driving: on the one hand we want our autonomous vehicles (AVs) to be safe, while at the same time their performance (\eg comfort or progression) is key to adoption.
To effectively walk the tight-rope between safety and performance, AVs need to be risk-averse, but not entirely risk-avoidant. 
To facilitate safe-yet-performant driving, in this paper, we develop a \emph{\taskaware} risk estimator that assesses the risk a perception failure poses to the AV's motion plan.
If the failure has no bearing on the safety of the AV's motion plan, then regardless of how egregious the perception failure is, our \taskaware risk estimator considers the failure to have a low risk; on the other hand, if a seemingly benign perception failure severely impacts the motion plan, then our estimator considers it to have a high risk.
In this paper, we propose a \taskaware risk estimator to decide whether a safety maneuver needs to be triggered.
To estimate the \taskaware risk, first, we leverage the perception failure ---detected by a perception monitor--- to synthesize an alternative \plausible model for the vehicle's surroundings.
The risk due to the perception failure is then formalized as the ``relative" risk to the AV's motion plan between the \perceived and the alternative \plausible scenario.
We employ a statistical tool called \emph{copula}, which models tail dependencies between distributions, to estimate this risk. The theoretical properties of the copula allow us to compute probably approximately correct (PAC) estimates of the risk.
We evaluate our \taskaware risk estimator using NuPlan and compare it with established baselines, showing that the proposed risk estimator achieves the best $F1$-score (doubling the score of the best baseline) and exhibits a good balance between recall and precision, \ie a good balance of safety and performance.
\end{abstract}
\IEEEpeerreviewmaketitle


\section{Introduction}


Despite the fast-paced progress in robotics and autonomous systems, 
perception modules in autonomous vehicles (AVs) still encounter a spate of failure modes (e.g., misclassification or misdetection of objects, ghost obstacles, out-of-distribution (OOD) objects, etc.), which can compromise the safety of passengers, other drivers, and pedestrians.
Consequently, the problem of developing detectors for such perception failures has recently gained traction~\cite{antonante21iros-perSysMonitoring, sever22-hjreachability, Miller22arxiv-falseNegativeObjectDetectors, bogdoll22cvpr-anomalyAD}.
However, these failures occur frequently enough that reverting to a fallback safety maneuver for each such detection is prohibitively detrimental to the performance of the AV. 
In this paper, we work towards developing a \emph{\taskaware perception monitor} that only triggers when the perception failure poses a significant risk to the AV's motion plan, thereby, promoting safe yet performant driving; an example highlighting the importance of developing a \taskaware perception monitor, which focuses on task-relevant perception failures, 
is illustrated in \cref{fig:task-relevant-perception}.

We envision a \taskaware perception monitor that embodies three main components, as shown in \cref{fig:anchor}. 
First, the \emph{perception failure detection and identification module} identifies perception faults and isolates the responsible modules and failure modes.
Second, the \emph{\hypgen} leverages the knowledge of the perception failure modes, provided by the failure identification, to construct a probabilistic (possibly multi-modal) description of plausible alternative models for the AV's surroundings that supports the actual world scene.\footnote{The probabilistic description of \plausible AV surroundings might be highly stochastic and multi-modal. Planning in the \plausible scene would be impractical and possibly not conducive to a good plan; however, we can still leverage it to estimate the risk of the perception failures to the AV using the approach we develop in this paper.}
Finally, a \emph{\taskaware risk estimator} assesses the increased risk to the AV's motion plan due to the perception failure.
There is a plethora of recent work on perception failure detection~\cite{antonante21iros-perSysMonitoring,ramanagopal2018failing,Miller22arxiv-falseNegativeObjectDetectors,bogdoll22cvpr-anomalyAD}, and also some work on \hypgeneration \cite{christianos22arxiv-planning,itkina22icra-occlusionInference,Antonante22arxiv-perceptionMonitoring}, comparably much less work on \taskaware risk estimation. 
The primary subject of this paper is, indeed, the development of the \taskaware risk estimator. 

\begin{figure}[t]
  \centering
  \subfigure[\Taskrelevant failure]{
    \label{fig:task-relevant}\includegraphics[width=0.45\columnwidth]{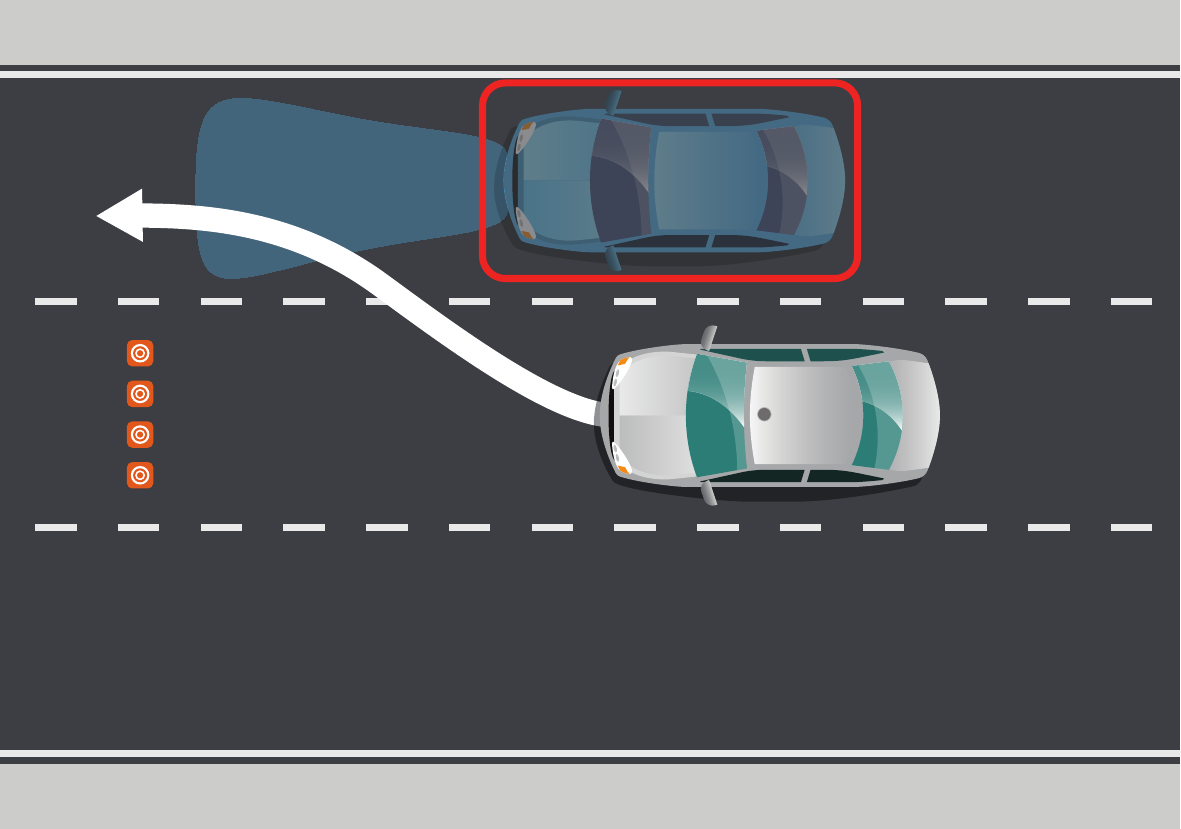}
  }
  \subfigure[Non-\taskrelevant failure]{
    \label{fig:non-task-relevant}\includegraphics[width=0.45\columnwidth]{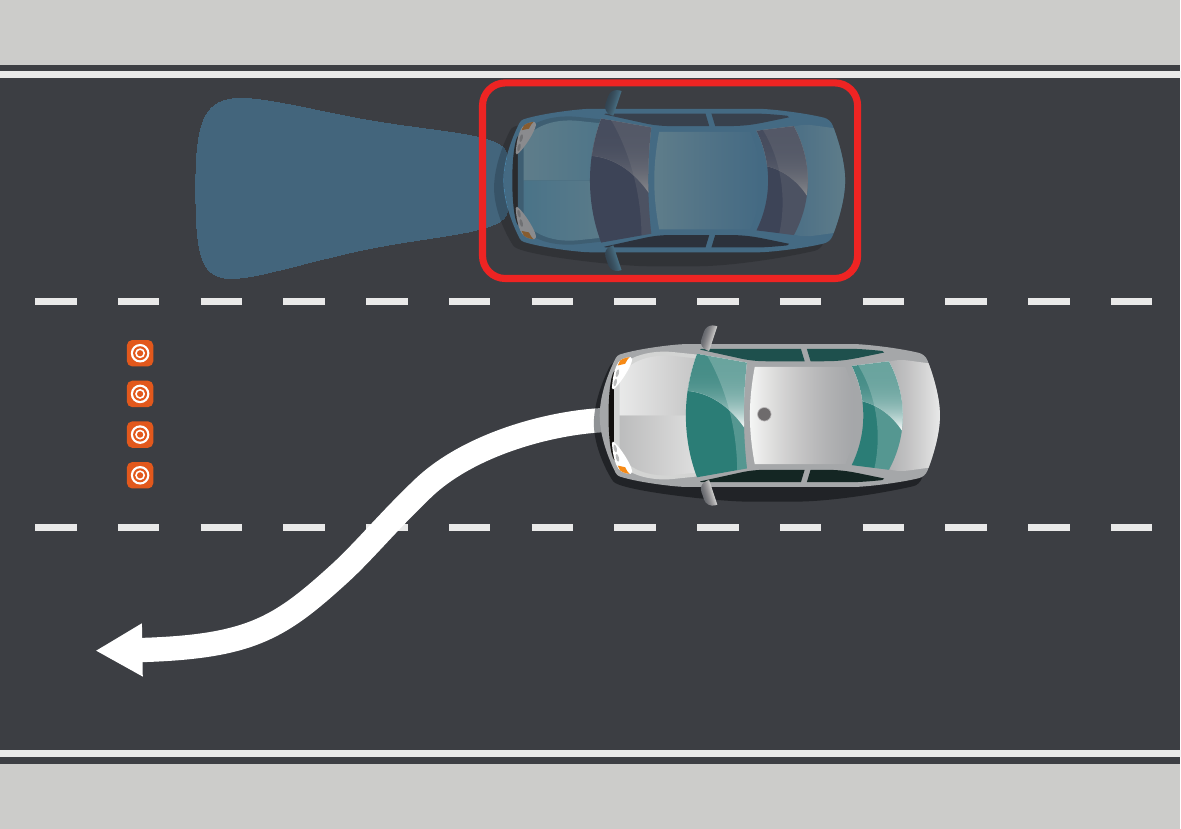}
  }
  \caption{
    {\bf Illustration of \taskaware perception failure detection.} 
    The white car is the ego vehicle and the blue car is an external (non-ego) vehicle.
    In this example, the non-ego vehicle has not been detected by the perception system of the ego vehicle. Then, \cref{fig:task-relevant} depicts a \taskrelevant missing obstacle, as the ego vehicle's motion plan will likely collide with the non-ego vehicle due to the misdetection. 
    \cref{fig:non-task-relevant} depicts a non-\taskrelevant missing vehicle, as the ego vehicle's motion plan will not lead to a collision with the non-ego vehicle, regardless of the perception failure.
  }
  \label{fig:task-relevant-perception}
  \vspace{-5mm}
\end{figure}


\begin{figure*}[htbp]
  \centering
  \includegraphics[width=\textwidth]{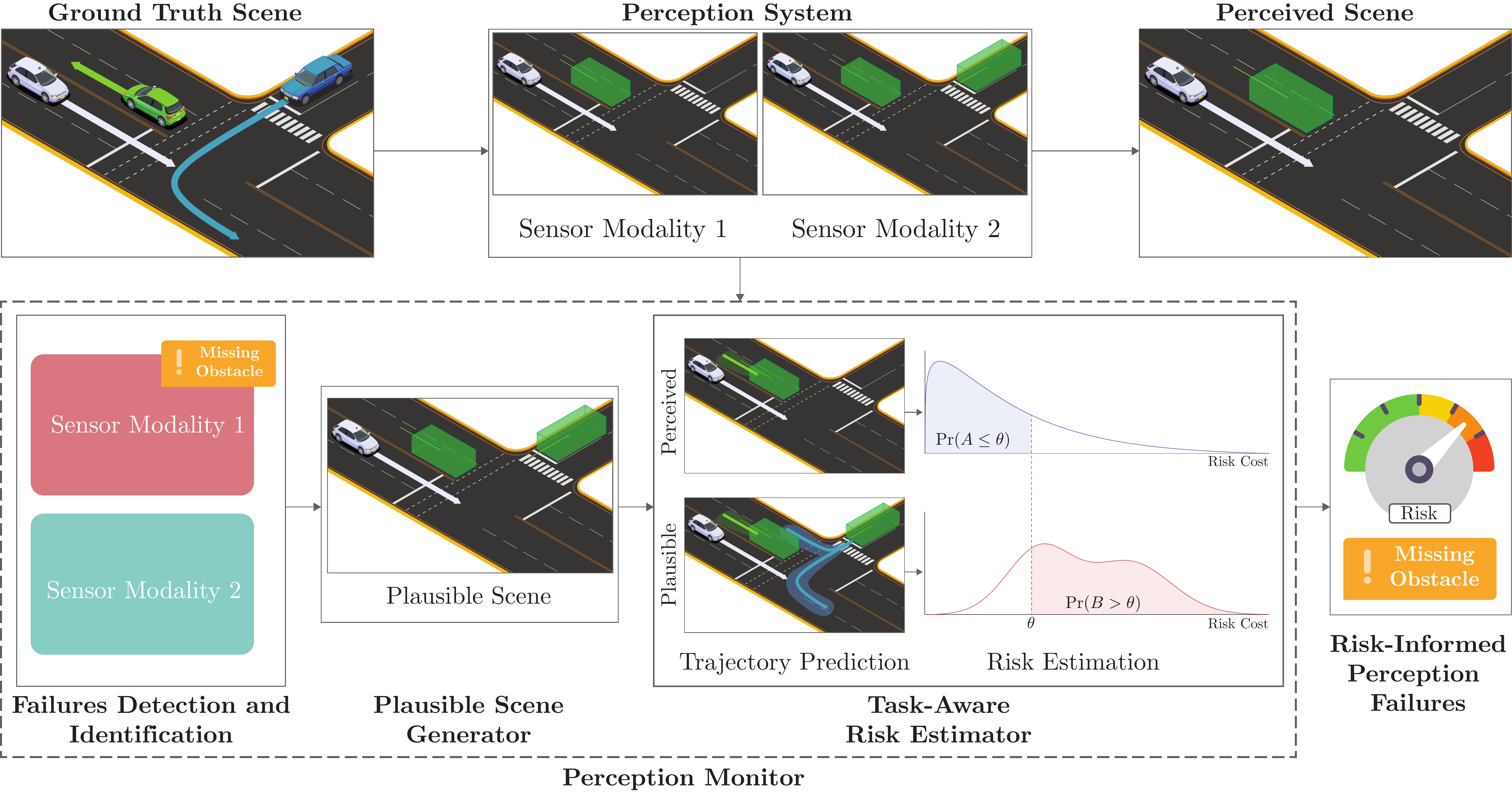}
  \vspace{-6mm}
  \caption{{\bf \Taskaware perception monitor overview.} 
  The scene contains the ego vehicle (white car) and two non-ego agents (green and blue car). 
  The top row shows a scenario in which a perception system fails to detect an obstacle (the blue car):
  one of the two sensor modalities used by the perception system is not able to detect the obstacle (top-center subfigure), inducing a missing-obstacle failure in the perception output (top-right subfigure). 
  The bottom row depicts the proposed task-aware perception monitor.
  The failure detection and identification module detects that sensor 1 is failing (for example using spatio-temporal information).
  The \hypgen, uses the information about the active failures, generates a \plausible scene from the \perceived scene.
  Finally, the task-aware risk estimator computes the risk associated with the failure.
    The shaded (green and blue) regions in the bottom-row scenes represent the uncertainty in the trajectories, as computed by the non-ego trajectory prediction module.
    The possible trajectories induce a distribution of risk costs for each scene, which are used to estimate the risk associated with a perception failure.
    If the risk in the \plausible scene is significantly higher than the risk in the \perceived scene, we detect the failure as task relevant.
    Our detector uses a statistical tool called copula to estimate the tail dependency between the two cost distributions.
  }
  \label{fig:anchor}
  \vspace{-3mm}
\end{figure*}

The \taskaware risk estimator we develop in this paper compares the risk to the AV's motion plan in the \perceived scenario with the one in the generated \plausible scenarios.
The risk posed to the AV's motion plan in both the scenes (\perceived and \plausible) is expressed as a probability distribution on a risk metric, \eg time-to-collision. 
We introduce the notion of \emph{relative scenario risk} (RSR), which measures the probability that the \plausible scene has a high risk when the \perceived scene does not. 
To empirically estimate RSR, we employ the statistical tool called \emph{copula}~\cite{nelsen07boo-introToCopulas}, which models tail dependencies between distributions, and we provide probably approximately correct (PAC) bounds on the RSR estimate. 
Finally, we provide a detection algorithm based on the RSR PAC bounds that, with high probability, triggers an alarm when faced with high-risk \taskrelevant failures.


%
%

\myParagraph{Statement of Contributions} Our contributions in this paper are as follows: 
(i) We formalize the notion of \emph{relative scenario risk} (RSR), which underlies our \taskaware risk estimation;
(ii) We develop an algorithm to estimate RSR at runtime by leveraging the copula and also provide probabilistic guarantees on the correctness of our estimation; 
(iii) Finally, we demonstrate the efficacy of our framework by comparing our method with prior approaches on a dataset of \num{100} realistic perception failure scenarios created in NuPlan~\cite{nuplan}. 
We show that our risk estimator achieves the highest F1 score, exhibits a good balance of precision and recall, and anticipates collisions with sufficient time to allow mitigation measures.
\edit{The code is available at \url{https://github.com/NVlabs/persevere}}.


\section{Related Work}\label{sec:related_work}

We discuss prior art for all of the three stages of our perception monitoring scheme, \ie perception failure detection, \hypgeneration, and task-relevant risk estimation.

\myParagraph{Perception Failure Detection and Identification}
Autonomous vehicles rely on onboard \emph{perception systems} to provide situational awareness and inform the onboard decision-making and control systems.
Reliability of the perception system is critical for safe operation of AVs.
While it is desirable for the perception system to be fault-free under any conditions, it is hard to guarantee it~\cite{bogdoll22cvpr-anomalyAD};
therefore, detection and identification of failures in the perception system at runtime have gained increasing attention.
The problem of fault detection and identification is studied in~\cite{Antonante22arxiv-perceptionMonitoring} where the authors proposed a system-level framework for online monitoring of the perception system of an AV. 
Besides failure detection, the framework in \cite{Antonante22arxiv-perceptionMonitoring} also identifies, at runtime, the failure modes that the system is experiencing, from an a priori known list of failures.
Other approaches in the literature include spatio-temporal information from motion prediction to assess 3D object detection \cite{You21-temporalCheckLiDAR}, Timed Quality Temporal Logic (TQTL) to reason about desirable spatio-temporal properties of a perception algorithm \cite{Balakrishnan19date-perceptionVerification, Balakrishnan21-percemon}, or detect anomalies by placing logical-constraint model assertions \cite{Kang18nips-ModelAF}.
Since perception systems are composed of multiple modules, failure detection for specific submodules has also received attention.
Previous works focused on object detection~\cite{Miller22arxiv-falseNegativeObjectDetectors}, semantic segmentation~\cite{Besnier21cvf-triggeringFailures, Rahman22ral-FSNet}, localization~\cite{Jing22tits-GPSintegrity, Hafez20ral-integrityMonitoring}, out-of-distribution (OOD) detection \cite{Sharma21uai-scod,hendrycks2019scaling}, or changes in high-definition map~\cite{lambert21neurips-trustButVerify}.
All of the above works focus on detecting and identifying failures in the perception system, but do not assess their impact on the AV's motion plan.

\myParagraph{\HypGeneration} 
While there is limited prior work on explicit \plausible scene generation, many works in the literature reason about plausible alternative scenes in order to detect or avoid failures.
Indeed, failure detection methods often use spatio-temporal inconsistencies between different sensor modalities, where each modality implicitly proposes a \plausible scene.
You\setal~\cite{You21-temporalCheckLiDAR} use historical information from a number of previous 3D scenes to predict a \plausible scene, that is then compared with the \perceived scene to detect errors.
Similarly, \cite{Liu21tdsc-detectingPerceptionError} correlates camera images and LIDAR point clouds to detect missing (or ghost) obstacles.
Beside obstacle detection, previous work also tackled the mislocalization error. Li\setal~\cite{li18itiv-mapDeadReckoning} use particle filters to retain multiple likely positions of the AV.
Futhermore, the literature on planning under occlusions contains both model-based~\cite{galceran15iros-augmentedSceneOcclusion, koschi20tiv-setBasedOcclusion, zechel19irc-pedestrianOcclusion} and learning based approaches~\cite{christianos22arxiv-planning,itkina22icra-occlusionInference,han21corl-planning}.
These works augment the scene to include possible missing (occluded) obstacles in the scene.
The same techniques can be used to generate \plausible scenes whenever the failure detection does not provide enough information about the \plausible scene.

\myParagraph{Risk Assessment}
There are several risk assessment techniques in the literature~\cite{dahl18tiv-collisionAvoidance, westhofen23-criticalityMetrics}.
One approach to ``measure" the risk is to monitor the deterioration of the cost of the motion plan, as was done in \cite{farid2022failure,farid2022task,farid22arxiv-predictionAnomaly}. 
Another approach is to use a \emph{criticality metric}, such as Time To Collision (TTC), which computes the time before a collision happens between two actors if their speeds and orientations remain the same. There is a large variety of criticality metrics in the literature, and the interested reader is referred to~\cite{westhofen23-criticalityMetrics} for a comprehensive overview; however,
these metrics, in general, only assess whether a scenario is dangerous while assuming that the inferred scene from perception is correct. 
A recent approach uses a neural network to classify the risk level using labeled data \cite{agarwal2022risk}. 
However, learning-based methods suffer from lack of OOD robustness, and are usually less interpretable and lack guarantees. 
Other works on perception-aware risk assessment, such as \cite{bansal2021risk}, which proposes risk-ranked recall for object detection systems, and \cite{bernhard2022risk}, which develops perception-uncertainty-based risk envelopes, do not reason about the future actions of other agents. 
Topan \etal \cite{sever22-hjreachability} do account for the future actions of other agents for perception failures by computing ``inevitable'' collision sets (ICS) via Hamilton-Jacobi (HJ) reachability analysis~\cite{mitchell05tac-HJReachability}, and flagging a situation as unsafe if another agent is close to entering the ICS~\cite{leung20ijrr-HJReachability, hsu22arxiv-sim2real,hu22ral-sharp}. 
However, \cite{sever22-hjreachability} assumes the worst-case scenario will occur, \ie the AV and other road agents will try to collide with each other, and does not consider the interactions between multiple agents, resulting in over-conservatism. 
In this paper, we estimate the risk to the AV's motion plan in the presence of a perception failure while accounting for future actions of other agents  by leveraging a trajectory prediction network \cite{salzmann20eccv-trajectron}. 
Furthermore, our approach is accompanied with PAC bounds and is run-time capable. 



\section{\TaskAware Perception Monitors: Overview and Risk Estimation Formulation}\label{sec:problem-formulation}

This section provides an overview of the building blocks of our \taskaware perception monitor, which comprises of three components: perception failure detection and identification, \hypgen, and \taskaware risk estimator. Moreover, the section formalizes the problem of \taskaware risk estimation, which is the main focus of this paper.
Throughout the rest of the paper we will refer to the AV as the \emph{ego vehicle} while any other agent is referred to as a \emph{non-ego agent}.

\myParagraph{Perception Failure Detection and Identification}
We assume access to a perception failure detection and identification module, such as the algorithms presented in \cite{Antonante22arxiv-perceptionMonitoring}, that identifies the set of failure modes the perception system is experiencing.
The perception system is composed of a set of modules, each of which is responsible for a specific task, e.g., object detection, localization, etc.
Each module is subject to a finite set of failure modes.
The perception failure detection and identification module computes a failure state vector $\faults$, containing the relevant information about the active failures, that is, the set of active failure modes and the corresponding perception diagnostic information (such as intermediate detection results, raw sensor data, etc.).

\myParagraph{\HypGen} 
Before we assess the risk that the perception failure poses to the AV's motion plan, we need to understand the actual scene in which the AV is operating.
To this end, we construct a new estimate of the surrounding scene using what we call a \emph{\hypgen}. 
Let $\statevar_t\in\statespace\subseteq \mathbb{R}^n$ be an estimate of the world state at time $t$, provided by the perception module (this is the ``perceived scene'').
We assume that the world state comprises the ego vehicle's state $\statevar^{\rm e}$, non-ego agents' states $\statevar^{\rm ne}$, and map attributes $\statevar^{\rm m}$, e.g., lane lines, stop signs, traffic signals, etc.
Given the \perceived world state $\statevar_t$ from the perception module, and the active perception failure modes information $\faults$ from the perception failure detection and identification, the \hypgen returns alternative \plausible scenes in the form of a probability distribution $\hypdistr(\hat{\statevar}_t|\statevar_{0:t},\faults)$ over the \plausible world states $\hat{\statevar}_t\in\statespace$ at time $t$; we require the \hypgen to \emph{support the actual world state}.

While there are some approaches that can be used for scenario generation \cite{Antonante22arxiv-perceptionMonitoring,You21-temporalCheckLiDAR,christianos22arxiv-planning,itkina22icra-occlusionInference, han21corl-planning}, the approach we adopt here is based on the perception failure identification method in~\cite{Antonante22arxiv-perceptionMonitoring}.
In a nutshell, the method in~\cite{Antonante22arxiv-perceptionMonitoring} detects inconsistencies between intermediate perception results and infers the set of faults that could have caused such inconsistencies.\footnote{The method~\cite{Antonante22arxiv-perceptionMonitoring} can also integrate more complex tests, including logic formulas, mathematical certificates of correctness, and a priori bounds.}
For instance, consider the case where the radar-based detection module detects an obstacle in front of the ego-vehicle, but the same obstacle is missed by the camera-based detection module: in this case the perception system might prioritize camera detection and discard the radar detection as a false detection; therefore, the \perceived scene would have no obstacle in front of the ego-vehicle.
However, the approach in~\cite{Antonante22arxiv-perceptionMonitoring} can detect the inconsistency between the radar and the camera and, making spatio-temporal considerations, can infer that the radar detection is indeed correct.
The \hypgen can then use the information about the failure to generate an alternative \plausible scene that contains an obstacle as detected by the radar.
In general, the alternative scene may not be unique. 
For example, since the radar-based detection module may only be able to detect the position and the velocity of the missing obstacle, but not its class (e.g., car, pedestrian, etc.),
such a detected failure may give rise to a \textit{probability distribution} over \plausible scenes
$\hypdistr(\hat{\statevar}_t|\statevar_{0:t},\faults)$ (e.g., the uniform distribution over the missing obstacle's class). Similarly, this distribution $\hypdistr(\hat{\statevar}_t|\statevar_{0:t},\faults)$ can also model uncertainty in the radar position and velocity measurements.
Assuming that at least a module in the perception pipeline computed the correct result, the distribution of generated scenes will support the actual scene.\footnote{This approach can also be generalized to early-fusion and middle-fusion perception systems as spatio-temporal information across frames can provide useful diagnostic information regardless of the perception architecture.}
In this paper, we do not assume a particular implementation of the \hypgen, which might depend on the perception system architecture, but we rather focus on the risk assessment.
We only require the \plausible scene distribution to support the actual scene, which has been shown to be possible for the most common perception failure modes~\cite{Antonante22arxiv-perceptionMonitoring, You21-temporalCheckLiDAR, christianos22arxiv-planning, Liu21tdsc-detectingPerceptionError, galceran15iros-augmentedSceneOcclusion} (as discussed in~\cref{sec:related_work}).

\myParagraph{Relative Scenario Risk} We are interested in understanding how much more risk does the ego's motion plan encounter in the generated \plausible scenes $\hypdistr(\hat{\statevar}_t|\statevar_{0:t},\faults)$ compared to the \perceived scene $\statevar_{0:t}$.
Let $\statevar^{\rm e}_{t:t+T}$ be the ego's motion plan generated by a planning module for a time horizon $T$. 
We assume the availability of a trajectory prediction module which provides a distribution $\psi(\statevar_{t:t+T}|\statevar_{0:t},\statevar^{\rm e}_{t:t+T})$ on the future world state trajectories conditioned on the world state history and the ego's motion plan. 
\begin{edited}
  The trajectory predictor $\psi$ reasons about agent interactions and provides multimodal predictions that account for multiple agent intentions; in our experiments we use Trajectron++~\cite{salzmann20eccv-trajectron} which is a state-of-the-art trajectory prediction model that satisfies these criteria.
\end{edited}
The approach we describe is agnostic to the choice of the planning and prediction modules.
Let $c:\statespace\to\Reals{}^+$ be a \emph{cost function} such that higher values imply riskier scenarios for the ego vehicle.
Examples of such functions might be the distance between the ego vehicle and the closest non-ego agent or a surrogate (so that higher values imply smaller times) of the time-to-collision metric~\cite{westhofen23-criticalityMetrics}. 
The distribution $\psi$ on the future world states $\statevar_{t:t+T}$ induces a sequence of univariate distributions $\{\phi_{t+\tau}(c_{t+\tau}|\statevar_{0:t},\statevar^{\rm e}_{t:t+T})\}_{\tau=1}^T$ over the predicted costs $c$ for each time step in the future. 
In the rest of the paper, we will work with the predicted cost distribution $\phi_{t+\tau}$ for a particular $\tau$ and for a particular motion plan $\statevar^{\rm e}_{t:t+T}$. 
For the sake of notational compactness, we drop the explicit dependence of $\phi$ on $t+\tau$ and $\statevar^{\rm e}_{t:t+T}$ to express the predicted cost distribution as $\phi(c|\statevar_{0:t})$. 
Similarly, the \plausible scene distribution $\hypdistr(\hat{\statevar}_t|\statevar_{0:t},\faults)$ on the \plausible world state $\hat{\statevar}_t$ induces the cost distribution $\phi\left(c|\statevar_{0:t},\faults\right)$.

We are now ready to formalize our \taskaware notion of risk in the following definition.
\begin{definition}[Relative Scenario Risk (RSR)]\label{def:relative-scenario-risk}
  Let $\statevar_{0:t}$ be the world state history for the \perceived scene and let $\hypdistr(\hat{\statevar}_t|\statevar_{0:t},\faults)$ be the distribution of the generated \plausible scenes due to the perception faults $\faults$, detected by a perception failure detection and identification module. 
  Let $\phi_A:=\phi(c|\statevar_{0:t})$ be the distribution of the costs for the \perceived scene and $\phi_B:=\phi\left(c|\statevar_{0:t},\faults\right)$ be the distribution of the costs for the \plausible scenes. 
  Let $\costth\in\mathbb{R}^+$ be the cost threshold that the planner desires to stay below. 
  The \emph{relative scenario risk} (RSR) between the \plausible and the \perceived scenes is then defined as:
  \begin{align}\label{eq:RSR}
    \hat{\riskfcn}: \costth \mapsto \underset{A\sim\phi_A, B\sim\phi_B}{\Pr}(B > \costth \mid A \leq \costth).
  \end{align}
\end{definition}
For a given $\theta$, the higher the RSR is, the further the \plausible scene cost distribution $\phi_B$ is skewed towards higher costs, as illustrated in~\cref{fig:anchor}.
Hence larger values of $\hat{\riskfcn}$ imply that the \plausible scenes are riskier than the \perceived scene. 
Note that, in the general, $\phi_A$ and $\phi_B$ will not be independent as the underlying scene is largely the same.


The choice of $\theta$ defines a desired safety threshold: for instance, if the cost is the distance between agents, $\theta$ can be the smallest acceptable distance between the ego and the nearest agent. 
The choice of $\theta$ may be scenario dependent (e.g., the minimum distance might be different when driving on a highway vs. a traffic jam).
To overcome this dependency, we take $\costth$ to capture the bulk of the probability mass in the \perceived scene cost distribution ($\phi_A$).
To make this more concrete, let $\Phi_A$ be the marginal cumulative distribution function (CDF) of $\phi_A$ and $\Phi_B$ be the CDF of $\phi_B$. Recall that the generalized inverse of a CDF $\Phi$, here denoted by $\Phi\inv$, is defined as:
\begin{equation}\label{eq:phi-inv}
  \Phi\inv(p) := \inf \left\{ c\in\Reals{} \>:\> \Phi(c) \geq p \right\}.
\end{equation}
Then, we choose $\costth=\Phi_A\inv(p)$.
This is equivalent to taking $\costth$ to be the maximum value of the risk cost, among the most common situations in the \perceived scene.
We call $p$ the \emph{risk aversion} parameter as it denotes the amount of risk the ego agent is willing to accept in its motion plan.

We can now define the $p$-quantile relative scenario risk.
\begin{definition}[$p$-quantile Relative Scenario Risk]\label{def:p-RSR}
    Let $\hat{\riskfcn}$ be the relative scenario risk in \cref{def:relative-scenario-risk}.
    Let $p\in(0,1)$ be the risk aversion parameter described above.
    Then, the $p$-quantile relative scenario risk ($p$-RSR) is defined as:
    \begin{align}\label{eq:p-RSR}
        \riskfcn:p \mapsto \hat{\riskfcn}\circ \Phi_A\inv(p) .
    \end{align}
\end{definition}
Note that the definition above is simply restating~\cref{def:relative-scenario-risk} in terms of the risk aversion parameter $p$.

\myParagraph{Problem Statement} Given the \perceived world state history $\statevar_{0:t}$, the perception module fault modes $\faults$, and risk aversion $p$, we want to estimate the $p$-quantile relative scenario risk $\riskfcn(p)$ in \cref{def:p-RSR}.
It is worth pointing out that this is a challenging problem because the distributions $\phi_A$ and $\phi_B$ are not independent, and we do not have an explicit analytical representation for them or their CDFs $\Phi_A$, $\Phi_B$; hence, we cannot analytically compute $\Phi_A\inv(p)$ or $\riskfcn(p)$.
However, we can sample from these distributions independently.
In particular, we take samples from these distributions in a way that they are independent (any new sample does not depend on the previous) and identically distributed (the underlying scene and behavior of the agents is fixed).
In the next section, we present a method to estimate $\riskfcn(p)$ using these samples.


\section{Task-Aware Risk Estimation}

In this section, we use \emph{copulas}~\cite{joe14book-copulas}, a statistical tool used to model tail dependencies between distributions, to provide an algorithm to estimate the RSR defined  in \cref{def:p-RSR}.

\subsection{Introduction to Copulas}
To model the dependency between the two univariate distributions $\phi_A$ and $\phi_B$, we use the concept of \emph{copula} (for a more extensive introduction see~\cite[Chapter 1]{joe14book-copulas}). 
Copulas are tools for modelling dependence of several random variables, and the name ``copula” was chosen to emphasize the manner in which a copula couples a joint distribution function to its univariate marginals.
We make this mathematically precise in the following definition.
\begin{definition}[Copula~\cite{nelsen07boo-introToCopulas}]\label{def:copula}
  A $d$-dimensional copula $\copula:[0,1]^d\to[0,1]$ is a function defined on a $d$-dimensional unit cube $[0,1]^d$ that satisfies the following:
  \begin{enumerate}
    \item $\copula(u_1,\ldots,u_{i-1},0,u_{i+1},\ldots,u_d) = 0$ for any $u_i$, $i\in\{1,\ldots,d\}$,
    \item $\copula(1,\ldots,1,u,1,\ldots,1) = u$ for any $u\in[0,1]$ in any position, and
    \item $\copula$ is $d$-non-decreasing.\footnote{
    That is, for each hyper-rectangle $B=\prod_{i=1}^d[x_i,y_i]\subseteq[0,1]^d$, the C-volume of $B$ is non-negative: $\int_B dC(u) = \sum_{z\in\prod_{i=1}^d \{x_i,y+i\}} (-1)^{N(\vzz)}C(\vzz) \geq 0$, where $N(\vzz)=\#\{k:z_k=x_k\}$}
  \end{enumerate}
\end{definition}
These three properties ensure that the copula behaves like a joint distribution function.
To gain intuition, consider each $u_i$ to be a probability in the range $[0,1]$.
The first condition says that if the probability of the event associated to $u_i$ is zero, then, regardless of the probability of the other events, the joint probability of all events happening at the same time is zero.
Conversely, if all events are sure to occur except one, then the probability of the joint event is the probability of the single non-sure event.
Finally, the last condition  imposes the copula to be non-decreasing in each component.

Sklar's theorem~\cite{sklar59-copula}, presented next, provides the theoretical foundation for the application of copulas
together with the conditions for the existence (and uniqueness) of the copula.
Note that in this paper, we only require the existence of the copula, but we include the uniqueness conditions as well below for the sake of completeness.
\begin{theorem}[Sklar's theorem~\cite{sklar59-copula}]\label{thm:sklar}
  Let $\Phi(x_1,\ldots,n_d)$ be a joint distribution function, and let $\Phi_i$, $i=1,\ldots,d$ be the marginal distributions.
  Then, there exists a copula $\copula:[0,1]^d\to[0,1]$ such that for all $x_1,\ldots,x_d$ in $[-\infty,+\infty]$
  \begin{equation}\label{eq:sklar}
    \Phi(x_1,\ldots, x_d) = \copula(\Phi_1(x_1), \ldots, \Phi_d(x_d)).
  \end{equation}
  Moreover, if the marginals are continuous, then $\copula$ is unique; otherwise, $\copula$ is uniquely determined on $\Range\Phi_1 \times \cdots \times \Range\Phi_d$ where $\Range\Phi_i$ denotes the range (image) of $\Phi_i$.
\end{theorem}
The importance of copulas in studying multivariate distribution is emphasized by Sklar's theorem, which shows, firstly, that all multivariate distribution can be expressed in terms of copulas, and secondly, that copulas may be used to construct multivariate distribution functions from univariate ones.
The latter point is particularly important for us because, as we noted in the previous section, we cannot sample from the joint distribution $\Pr(A,B)$, but we can sample from the marginals $\Pr(A)$ and $\Pr(B)$.


\subsection{Estimating $p$-RSR using Copula}


Let $A$, $B$ be random variables drawn from $\phi_A$ and $\phi_B$ with CDFs $\Phi_A$ and $\Phi_B$, respectively (notation introduced in~\cref{def:relative-scenario-risk}); as a quick reminder, $\phi_A$ is the cost distribution of the \perceived scene and $\phi_B$ of the \plausible scene. 
Let's assume for a moment that we can estimate the copula relating $\phi_A$ and $\phi_B$.
Since the copula $\copula(A,B)$ contains the information on the dependence structure between $(A,B)$, we can use it to measure the tail dependency between the two distributions.
Hence, using the definition of conditional probability and \cref{eq:sklar} from \cref{thm:sklar}, we can express $p$-RSR in \cref{eq:RSR} as follows:\footnote{For notational brevity, we are dropping the distributions from which the random variables $A$ and $B$ are drawn from under the probability sign $\Pr$.}

\begin{equation}\label{eq:copula-based-risk}
\begin{aligned}
  \riskfcn(p) & = \Pr({B > \Phi_A\inv(p) \mid A \leq \Phi_A\inv(p)}) \\
    &= 1 - \Pr(B \leq \Phi_A\inv(p) \mid A \leq \Phi_A\inv(p))\\ 
    &= 1- \frac{\Pr(A \leq \Phi_A\inv(p), B \leq \Phi_A\inv(p))}{\Pr(A \leq \Phi_A\inv(p))} \\ 
    &= 1- \frac{\copula(\Phi_A\circ \Phi_A\inv(p), \Phi_B\circ\Phi_A\inv(p))}{\Phi_A\circ\Phi_A\inv(p)} \\
    &= 1- \frac{\copula(p, \Phi_B\circ\Phi_A\inv(p))}{p}.
\end{aligned}
\end{equation}
Unfortunately, we do not have access to the explicit expression of the two CDFs $\Phi_A$ and $\Phi_B$ and the copula $\copula$, therefore, $\riskfcn$ cannot be computed analytically.
In what follows, we will provably bound $\riskfcn$ by constructing empirical estimates $\Phi_A\at{n}$ and $\Phi_B\at{n}$ of the CDFs $\Phi_A$ and $\Phi_B$, respectively, with $n$ i.i.d. samples from both, $\phi_A$ and $\phi_B$. 



\begin{theorem}[PAC bound on $p$-RSR]\label{thm:risk-bound}
  Let $\{A_i\}_{i=1}^n$ and $\{B_i\}_{i=1}^n$ be $n$ i.i.d. samples from CDFs $\Phi_A$ and $\Phi_B$, respectively. Let 
  \begin{align*}
      \Phi_A\at{n}(A) = \frac{1}{n}\sum_{i=1}^{n} \ind[A_i\leq A], & & \Phi_B\at{n}(B) = \frac{1}{n}\sum_{i=1}^{n} \ind[B_i\leq B]
  \end{align*}
  be empirical estimates of $\Phi_A$ and $\Phi_B$, respectively. 
  Let the risk aversion parameter $p\in(0,1)$ be as described in \cref{sec:problem-formulation} and let $\alpha\in(0,1)$. 
  Then, with probability at least $1-\alpha$:
  
  \vspace{-3mm}
  \small
  \begin{equation*}
    1-\frac{\min\{p,\upbound{v}(p,\alpha,n )\}}{p} \leq \riskfcn(p) \leq 1 - \frac{\max\{p+\lowbound{v}(p,\alpha,n )-1,0\}}{p},
  \end{equation*}
  \normalsize
  where
  \begin{equation*}
    \begin{aligned}
      \lowbound{v}(p,\alpha,n) &= \Phi_B\at{n}\circ\left[\Phi_A\at{n}+\epsilon(\alpha,n)\right]\inv(p) -\epsilon(\alpha,n)\\
      \upbound{v}(p,\alpha,n) &= \Phi_B\at{n}\circ\left[\Phi_A\at{n}-\epsilon(\alpha,n)\right]\inv(p) + \epsilon(\alpha,n)
    \end{aligned}
  \end{equation*}
  and $\epsilon(\alpha,n) = \sqrt{\ln(2/\alpha)/{(2n)}}$.
  \begin{proof}
    See \cref{app:proofs}.
  \end{proof}
\end{theorem}
Both bounds are sharp in the sense that they can be attained.
In particular, the lower bound is attained when $\phi_A$ and $\phi_B$ 
are \emph{perfectly positively dependent} in the sense that $B$ is almost surely a strictly increasing function of $A$.
Conversely, the upper bound is attained when $\phi_A$ and $\phi_B$
are \emph{perfectly negatively dependent}, meaning that $B$ is almost surely a strictly decreasing function of $A$.

The PAC bounds in~\cref{thm:risk-bound} are tractable to compute and allow us to estimate $p$-RSR $\riskfcn(p)$ at runtime. 
In particular the assumption on i.i.d. samples is not restrictive: as we noted in~\cref{sec:problem-formulation}, our problem formulation allows i.i.d. samples.
Also we do not assume a particular copula, but only its existence, which is proved by Sklar's theorem.


\subsection{Triggering Safety Maneuvers} 

With the results presented in~\cref{thm:risk-bound}, we have a way to measure if, compared to the \perceived scene, the \plausible scene exposes the ego-vehicle to unwanted risk in terms of probability.
However, if the probability is low, it might be detrimental for the ego-vehicle's performance to trigger safety maneuvers to mitigate the risk.
In the following we design a detection algorithm (\cref{alg:cap}) that can be used to detect if the system is likely to be experiencing a high-risk situation, which can be used to trigger safety maneuvers.

Consider a \emph{risk threshold} $\riskth\in(0,1)$ that denotes high-risk situations. 
If the lower bound on $\riskfcn(p)$ in~\cref{thm:risk-bound} is above $\riskth$, it means that with probability at least $1-\alpha$, the current scene indeed corresponds to   a high-risk situation (in the sense of $p$-RSR). In such a case, \cref{alg:cap} detects a task-relevant failure (line~\ref{line:true}), which can be used to trigger a safety maneuver.
This reasoning can be easily extended to consider multiple thresholds for multiple criticality levels, each associated with different mitigation strategies or different driving scenarios (\eg highway, urban driving, pickup/drop-off, etc.). 

\begin{algorithm}
\caption{$p$-RSR Detection Algorithm}\label{alg:cap}
\begin{algorithmic}[1]
\Require The state $\statevar_{0:t}$, the faults $\faults$, the cost metric $c$, the risk aversion $p$, the confidence level $1-\alpha$, and the risk threshold $\riskth$.
\Ensure $\ltrue$ if critical scenario, $\lfalse$ otherwise.
\State $\{A_i\}_{i=1}^n\sim\phi(c|\statevar_{0:t}),\> \{B_i\}_{i=1}^n\sim\phi(c|\statevar_{0:t},\faults)$ \label{state:sampling}
\State $\Phi_A\at{n}(A) \gets \nicefrac{1}{n}\sum_{i=1}^{n} \ind[A_i\leq a]$ \label{state:cdf_est_A}
\State $\Phi_B\at{n}(B) \gets \nicefrac{1}{n}\sum_{i=1}^{n} \ind[B_i\leq b]$\label{state:cdf_est_B}
\If{$\min\{p,\upbound{v}(p,\alpha,n )\} < p(1-\riskth)$} \label{state:test}
  \State \Return $\ltrue$ \label{line:true}
\Else
    \State \Return $\lfalse$
\EndIf
\end{algorithmic}
\end{algorithm}

\begin{figure*}[t]
  \centering
    \includegraphics[width=\textwidth]{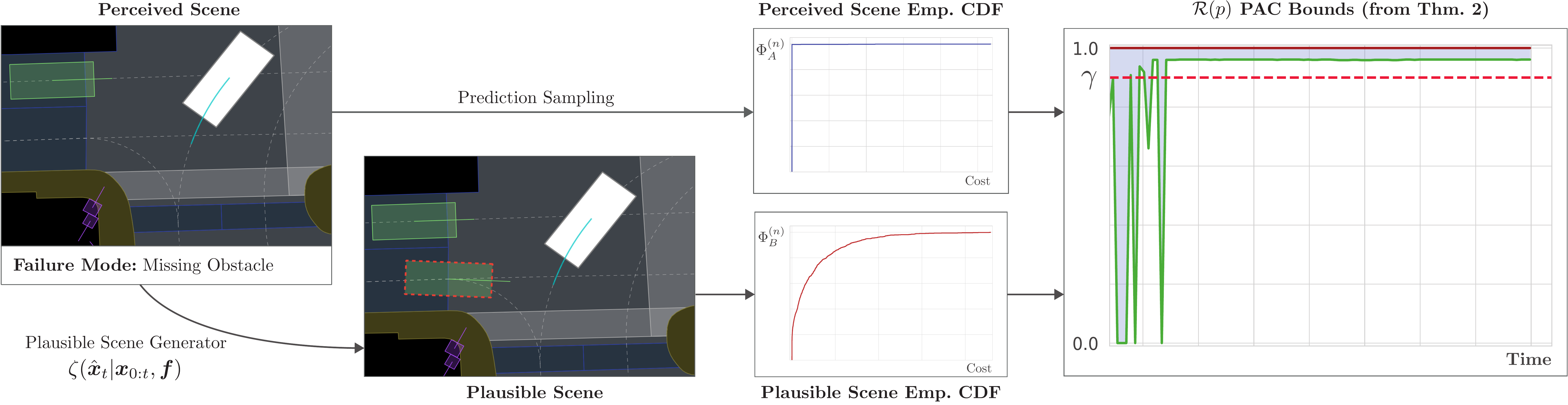}  
    \vspace{-6mm}
  \caption{
    {\bf Depiction of~\cref{alg:cap}.}
    The \perceived scene, which is subject to a missed-obstacle failure, is processed by the \hypgen which produces the \plausible scene.
    The two scenes are sampled and the empirical CDFs of the costs are estimated.
    The \perceived scene empirical CDF has a low risk since the only vehicle in the scene is stationary, since it is giving the ego vehicle the right-of-way.
    However, the \plausible scene has a higher risk since the ego vehicle is now in a collision path with a moving vehicle.
    The two CDFs are used to compute the PAC bounds in~\cref{thm:risk-bound}.
    The solid red line represent the upper bound, the solid green line the lower bound, while the dashed red line represent the risk threshold $\riskth$.
    Whenever the lower bound is above $\riskth$, the algorithm labels the scenario as high risk.
     \label{fig:algorithm}
     \vspace{-3mm}
  }
\end{figure*}

\Cref{alg:cap} can be summarized as follows:
after a failure is identified and the \plausible scene generated, the algorithm samples the two scenes (line~\ref{state:sampling}) and estimates the empirical CDFs (lines~\ref{state:cdf_est_A}-\ref{state:cdf_est_B}).
It then returns $\ltrue$ if the lower bound in~\cref{thm:risk-bound} is above the risk threshold $\riskth$, or $\lfalse$ otherwise (line~\ref{state:test}).
The algorithm steps are depicted in \cref{fig:algorithm}.

\begin{edited}
  The algorithm has four parameters: the risk aversion $p\in(0,1)$, the risk threshold $\riskth\in(0,1)$, the number of samples $n\in\Natural{}$, and the confidence level $1-\alpha\in(0,1)$.  
  The risk aversion $p$ measures the risk tolerance of the ego-vehicle, in terms of quantiles of the perceived scene risk distribution.
  Higher values indicate high risk tolerance, \ie the ego-vehicle is be expected to behave safely in riskier situations.
  Our risk metric $p$-RSR represents the probability that the plausible scene is riskier than the perceived scene.
  If this probability is significant, \ie above the risk threshold $\riskth$, the algorithm will classify the scene as high risk.
  To estimate $p$-RSR we use the PAC bounds in \cref{thm:risk-bound} requires the choice of a desired confidence level $1-\alpha$ and the number $n$ of predicted cost samples. 
  Clearly, $n$ should be as high as possible, but this is limited by the computational budget of the system.
  If the the values on $p$, $\riskth$, and $1-\alpha$ are close to $1$ the algorithm will be imprudent, \ie it will classify most of the scenes as low risk; on the other hand, if these values are small, the algorithm will be overly prudent, \ie it will classify most of the scenes as high risk.
  In our experiments, we found that the values of $p$, $\riskth$, and $1-\alpha$ in the range $[0.9, 0.99]$ provided a good trade-off between prudence and imprudence.
\end{edited}


\section{Experimental Results}

In this section, we compare the performance of our \taskaware risk estimator against various other baselines.
Our experiments were conducted on a desktop computer with an \texttt{Intel i9-10980XE} \SI{4.7}{\giga\hertz} CPU (36 cores) and an \texttt{NVIDIA GeForce RTX 3090} GPU.

\subsection{Dataset}
We tested the proposed approach using the publicly available NuPlan dataset~\cite{nuplan}.
To test the risk estimation we implement a fault injection mechanism into a NuPlan scenario.
We considered $4$ classes of failures, namely, \emph{Misdetection}, \emph{Missed Obstacle}, \emph{Ghost Obstacle}, and \emph{Mislocalization}.
Each of these classes is further divided into various subclasses.

\noindent\myParagraph{Misdetection} A misdetection represents an error in the estimation of one of the agents/objects around the vehicle. We consider:
\begin{inparaenum}[(i)]
    \item \emph{Orientation}: the ego perception system estimates the wrong orientation of the agent;
    \item \emph{Size}: the ego perception estimates the wrong size of an agent;
    \item \emph{Velocity}: the ego perception estimates the wrong velocity (both direction and/or magnitude) of the agent; and finally
    \item \emph{Traffic Light}: the ego perception estimates the wrong status for the traffic light.
\end{inparaenum}
All misdetection subtypes (except traffic light) are subject to noise, that can vary across scenarios.
For example, an orientation misdetection might offset the vehicle heading with a Gaussian distribution with mean $\pi/6$ and standard deviation of $0.1$.
\myParagraph{Ghost Obstacle} The ego perception system wrongly detects an obstacle that does not exist (i.e., a ghost obstacle). The ghost obstacle can be:
\begin{inparaenum}[(i)]
    \item \emph{in-path}, if it lays on the ego trajectory, or
    \item \emph{not in-path}, if it is not on the ego trajectory.
\end{inparaenum}
\myParagraph{Missing Obstacle} The ego perception system fails to detect an agent; the missed obstacle can be:
\begin{inparaenum}[(i)]
    \item \emph{in-path}, if it is on the ego trajectory, or
    \item \emph{not in-path}, if it is not on the ego trajectory.
\end{inparaenum}
\myParagraph{Mislocalization} The ego perception system fails to localize itself in the map.
Each failure mode can be:
\begin{inparaenum}[(i)]
    \item \emph{static} if the failure persists for the whole duration of a scenario (\SI{20}{\second}), or
    \item \emph{dynamic}, if it randomly appears/disappears over time.
\end{inparaenum}
In our experiments, a dynamic failure mode appears with probability $0.25$ and lasts at least \SI{1}{\second} before disappearing.

We manually designed $100$ realistic scenarios for evaluation, each with at least one common failure mode typically found in autonomous vehicles; see \cref{tab:scenarios} for a breakdown of the failures across scenarios.
Examples of such realistic scenarios include flickering detection of a pedestrian crossing the road, wrong orientation/velocity estimation of a vehicle with the right-of-way, misdetection of the traffic light with incoming traffic, etc.
For a detailed illustration of all the scenarios, please refer to the website \href{\datasetURL}{\datasetURL}.

\begin{table}
    \centering
    \caption{Scenarios}
    \label{tab:scenarios}
    \begin{tabular}{llcc}
        \toprule
        Failure Mode & Subtype & Static & Dynamic \\
        \midrule
        \multirow{2}{*}{Ghost Obstacle} & In-path & 5 & 5 \\
         & Not in-path & 10 & 10 \\
        \hline
        \multirow{2}{*}{Missing Obstacle} & In-path & 5 & 10 \\
          & Not in-path & 10 & 10 \\
        \hline
        \multirow{4}{*}{Misdetection} & Orientation & 10 & \\
         & Velocity & 10 & \\
         & Size & 5 & \\
         & Traffic Light & 5 & \\
        \hline
        Mislocalization & & 5 & \\
        \bottomrule
    \end{tabular}
\end{table}

\subsection{Implementation Details}
We implemented all the components of the proposed approach in Python.
As mentioned in \cref{sec:problem-formulation}, the proposed approach is planner agnostic.
In our experiments, we used the Intelligent-Driver Model (IDM) planner~\cite{treiber00-idm, albeaik22siam-limitationsIDM} provided in the NuPlan-devkit \cite{nuplan-devkit}.
The planner is designed to move towards the goal, following the lane, while avoiding collisions with the leading agent in front of the ego vehicle.
For the non-ego-prediction module, we instead used Trajectron++~\cite{salzmann20eccv-trajectron}.

To create high-risk situations, such as collisions, we use the closed-loop capability provided by NuPlan.
Closed-loop simulations enable the ego vehicle and other agents to deviate from what was originally recorded in the dataset by the expert driver.
In our simulations, each vehicle also behaves according to the IDM policy~\cite{treiber00-idm, albeaik22siam-limitationsIDM}.
However, due to a limitation of the NuPlan simulator, pedestrians and bicycles follow the original trajectory recorded in the dataset (open-loop).

\subsubsection{\HypGeneration}\label{sec:hyp-gen}
The primary goal of the experimental evaluation is to focus on the \taskrelevant risk-estimation. 
\edit{We use a \plausible scene generation method that, given the perception failure mode, proposes a \plausible scene by corrupting the ground-truth information (\ie velocity, size, orientation or location of the agents) with Gaussian noise.}
In particular, we add zero-mean Gaussian noise with standard deviation $0.2$m to the position, $0.1$rad to the heading, and $0.1$m/s to the velocity magnitude and direction.
This approach arises in the following common scenario.
Consider a perception system with two sensor modalities, \eg camera and radar, and a sensor fusion algorithm.
Suppose, without loss of generality, that the sensor fusion is misdetecting the \edit{velocity} of an agent due to a camera-based detection error, while the radar is fault-free.
\begin{edited}
  Once the fault detection and identification module recognizes the camera as the cause for the wrong \perceived scene, the \hypgen could use a Kalman filter to track the radar detections (non-failing sensor modality) to propose a plausible velocity of the vehicle.
  Since the Kalman filter produces a Gaussian estimate of the uncertainty, the velocity of the vehicle is also Gaussian. 
  This logic can be extended to other failure modes (\eg missing vehicle), and the \hypgen used in this paper emulates it.
\end{edited}
Generating \plausible scenes directly from a perception failure monitor and raw sensor data is beyond the scope of this paper; as discussed in \cref{sec:conclusion}, we will explore this topic further in our future work.

\subsubsection{Baselines}
We tested our approach against two baselines, one based on the Hamilton-Jacobi (HJ) reachability analysis~\cite{sever22-hjreachability} and another based on the collision probability.
\myParagraph{HJ Reachability} The core idea of HJ-Reachability is computing a set of target states that agents reason about either seeking or avoiding collision within a fixed time horizon.
There are two types of agent reactions in HJ-reachability, namely, \emph{collision-seeking} (min) and \emph{collision-avoiding} (max).
The idea behind the approach presented in~\cite{sever22-hjreachability} is to compensate for the lack of information about the perception failure by considering the conservative case in which both the agent and the ego vehicle are in a situation where the preferred actions are collision-seeking, namely the \emph{min-min} strategy.
For each agent in the scene, the HJ-reachability computes a value function (in our case based on the signed distance between the two bounding boxes), where its zero-sublevel set indicates the existence of a set of control inputs that lead to a collision.
The value function is pre-computed offline, and at runtime we perform look-ups, making this approach extremely fast.
Since there are multiple agents in the scene, we compute the value function for each agent and then we take the minimum value across the whole scene, if this value is smaller than zero, we say that the scene is high-risk.
\myParagraph{Collision Probability} The second baseline uses the trajectory prediction module to compute the probability of collision with any agents in the scene.
Analogous to our proposed approach, this baseline uses the \plausible scene to estimate the risk.
It samples the trajectories using the trajectory prediction module in both the \perceived scene and the \plausible scene, and if the collision probability in the \plausible scene is greater than the \perceived scene, and the former is above the threshold $\riskth$, we say that the scene is high-risk.
The key difference between our approach and the collision probability baseline is that the latter does not capture the dependency between the \perceived and the \plausible scene.
\begin{edited}
  Both baselines, Collision Probability and HJ-Reachability, use absolute risk thresholds that do not adapt to the scenarios. 
  In contrast, $p$-RSR measures the shift in the risk distribution due to the perceptual error between the perceived and plausible scenes, implicitly adapting the risk threshold to the scenario. 
  For example, suppose a correctly detected vehicle cuts into the ego vehicle's lane (high risk), but the speed of a distant cyclist (low risk) is underestimated.
  Both baselines consider this a high risk scenario because the correctly perceived vehicle is making a risky maneuver, even though the perception error is a low risk.
  Since the failure does not significantly shift the risk distribution, our approach correctly classifies it as low risk.
\end{edited}

\begin{table*}
  \centering
  \caption{Results}
  \label{tab:results}
  \resizebox{\textwidth}{!}{
  \begin{tabular}{cc|cccc|cc|cc}
    \toprule
    \multirow{2}{*}{Algorithm}                & \multirow{2}{*}{Parameters} & \multirow{2}{*}{F1 Score} & \multirow{2}{*}{Accuracy} & \multirow{2}{*}{Precision} & \multirow{2}{*}{Recall} & \multicolumn{2}{c|}{Alarm-to-Collision [\si{\second}]} & \multicolumn{2}{c}{Runtime [\si{\second}]}                                                 \\
                                              &                                &                           &                           &                            &                         & Average                                                & Median                                     & Average               & Median                \\
    \midrule
    \multirow{3}{*}{\shortstack[c]{Momentum-Shaped Distance\\(Proposed)}} & $p=0.90,\riskth=0.9,\alpha=0.1$                            & 0.70                      & 0.80                      & 0.55                       & \textbf{0.96}           & 5.28                                                   & 3.60                                       & \multirow{3}{*}{0.29} & \multirow{3}{*}{0.2}  \\
                                              & $p=0.95,\riskth=0.9,\alpha=0.1$                            & 0.79                      & 0.88                      & 0.68                       & \textbf{0.96}           & 5.18                                                   & 3.60                                       &                       &                       \\
                                              & $p=0.99,\riskth=0.9,\alpha=0.1$                            & \textbf{0.86}             & \textbf{0.93}             & \textbf{0.81}              & 0.92                    & 4.72                                                   & 3.03                                       &                       &                       \\
    \hline
    \multirow{3}{*}{\shortstack[c]{Time-To-Collision\\(Proposed)}}        & $p=0.90,\riskth=0.9,\alpha=0.1$                           & 0.70                      & 0.80                      & 0.55                       & \textbf{0.96}           & 4.61                                                   & 2.45                                       & \multirow{3}{*}{0.26} & \multirow{3}{*}{0.17} \\
                                              & $p=0.95,\riskth=0.9,\alpha=0.1$                           & 0.76                      & 0.86                      & 0.65                       & 0.92                    & 4.16                                                   & 2.05                                       &                       &                       \\
                                              & $p=0.99,\riskth=0.9,\alpha=0.1$                           & 0.79                      & 0.90                      & 0.79                       & 0.79                    & 4.13                                                   & 2.25                                       &                       &                       \\
    \hline
    \multirow{3}{*}{Collision Probability}    & $\riskth=0.90$                           & 0.40                      & 0.32                      & 0.26                       & \textbf{0.96}           & 6.53                                                   & 4.95                                       & \multirow{3}{*}{0.25} & \multirow{3}{*}{0.17} \\
                                              & $\riskth=0.95$                           & 0.41                      & 0.33                      & 0.26                       & \textbf{0.96}           & 6.31                                                   & 4.95                                       &                       &                       \\
                                              & $\riskth=0.99$                           & 0.43                      & 0.41                      & 0.28                       & 0.92                    & 4.33                                                   & 3.48                                       &                       &                       \\
    \hline
    HJ-Reachability                           & $\riskth=0$                           & 0.39                      & 0.28                      & 0.24                       & \textbf{0.96}           & \textbf{7.10}                                          & \textbf{5.75}                              & \textbf{0.01 }        & \textbf{0.01 }        \\
    \bottomrule
  \end{tabular}
  }
\end{table*}

\begin{table}[htbp]
\centering
\makegapedcells
\begin{tabular}{cc|cc}
  \multicolumn{2}{c}{}
   & \multicolumn{2}{c}{Predicted}                        \\
   &                               & High-Risk & Low-Risk \\
  \cline{2-4}
  \multirow{2}{*}{\rotatebox[origin=c]{90}{Actual}}
   & High-Risk                     & \shortstack{22                \\ (True Positive)}       & \shortstack{2 \\ False Negative}        \\
   & Low-Risk                      & \shortstack{5                \\ False Positive}         & \shortstack{71 \\ True Negative}       \\
  \cline{2-4}
\end{tabular}
\caption{Momentum-Shaped Distance Confusion Matrix}
\label{tab:msd_confusion_matrix}
\vspace*{-3em}
\end{table}

\subsection{Results}

In our experiments we use $n=1000$ samples from each scene, the confidence $1-\alpha$ and the risk threshold $\riskth$ are set to $0.9$.
We tested several values of risk aversion, namely, $p = \{0.9, 0.95, 0.99\}$, reported in~\cref{tab:results}.
We tested the time-to-collision~\cite{westhofen23-criticalityMetrics} and the Momentum-Shaped Distance cost metrics (described in \cref{app:cost_functions}) to assess the risk.
As mentioned before, the time-to-collision computes the time before a collision happens between two actors if their speeds and orientations remain the same.
The Momentum-Shaped Distance instead computes the distance between the bounding boxes of two actors, taking into account the relative velocity and orientation of the two actors.
The two metrics are described in greater details in~\cref{app:cost_functions}.
We consider a scenario to be high-risk if there is a collision.
This also allows us to compute the \emph{Alarm-To-Collision} metric, which measures the time between the first alarm raised by the perception monitor and the actual collision.


\cref{tab:results} reports the results averaged across the \num{100} scenarios. 
As mentioned above, we consider a scenario to be high-risk if there is a collision (ground-truth label),
and report the F1 score, precision, recall, accuracy, and the  \emph{Alarm-To-Collision} metric.
The proposed approach outperforms both baselines, \ie HJ-Reachability and collision probability, in terms of F1 Score, precision, recall, and accuracy.
In particular, our approach outperforms all others when using the Momentum-Shaped Distance with risk aversion $0.99$.
\begin{edited}
  Thus, while the baselines achieve similar performance on evidently risky situations (similar recall), our approach demonstrates greater finesse in subtle situations, increasing precision (and hence the F1 score). 
\end{edited}
Beside the relevant classification results, it anticipates the collision by an average of \SI{4.72}{\second}, giving enough time to the AV to take risk mitigation actions.
HJ-Reachability has the fastest runtime (recall that the value function is precomputed and, at runtime, the approach simply uses a lookup-table), but also the most conservative; it exhibits a high recall (on par with our approach) but the lowest precision, accuracy, and F1 Score.

\cref{tab:msd_confusion_matrix} reports the confusion matrix for the Momentum-Shaped Distance metric.
The table shows that our approach is able to detect both high-risk and low-risk scenarios reliably, with very few misclassifications.
It is worth noticing that several of the false positives result from situations in which there was a failure associated with an agent that is close to the ego vehicle, but the ego vehicle does not collide with it.


\myParagraph{Other Runtime Considerations}
The Momentum-Shaped Distance with risk aversion $0.99$ averages a runtime of \SI{0.29}{\second} and a median of \SI{0.2}{\second}.
The bottleneck of the proposed approach is the trajectory sampling, which is performed by the trajectory prediction module, in our case Trajectron++~\cite{salzmann20eccv-trajectron}.
From \cref{fig:timing} we can see that the trajectory prediction module takes \SI{0.22}{\second} on average, roughly \SI{75}{\percent} of the total runtime.
This limitation can be easily overcome by using a faster trajectory prediction module, such as PredictionNet~\cite{kamenev22icra-predictionnet}, which is two orders of magnitude faster than Trajectron++~\cite{kamenev22icra-predictionnet}.
Moreover, the proposed approach can be easily parallelized, as the cost computation can be computed in parallel for each agent in the scene and the two scenes can be batched into a single query for the prediction network.
With the suggested implementation improvements, we expect to achieve significantly faster runtimes.


\begin{figure*}
  \centering
  \subfigure{\includegraphics[width=59mm]{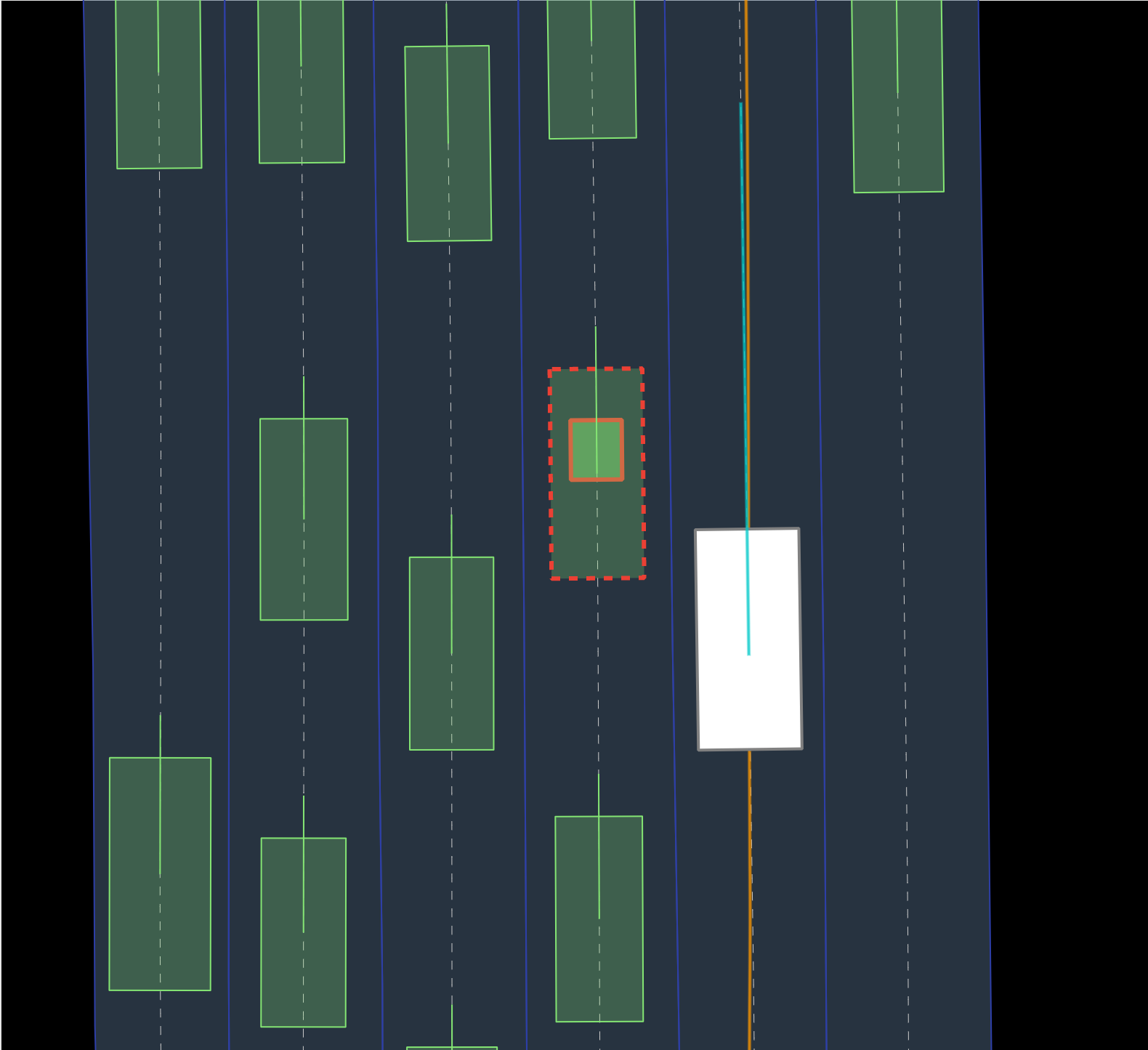}}
  \subfigure{\includegraphics[width=59mm]{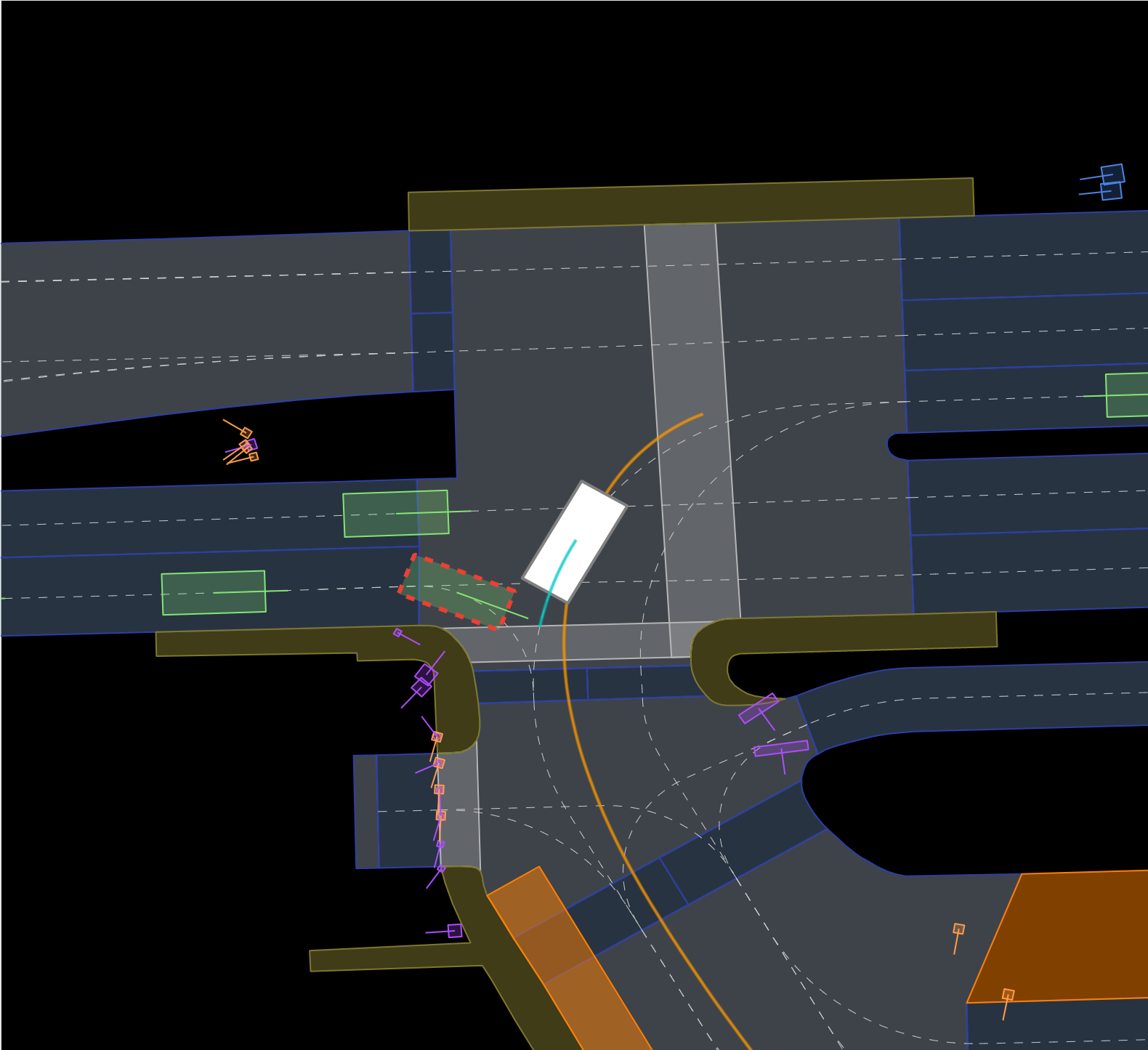}}
  \subfigure{\includegraphics[width=59mm]{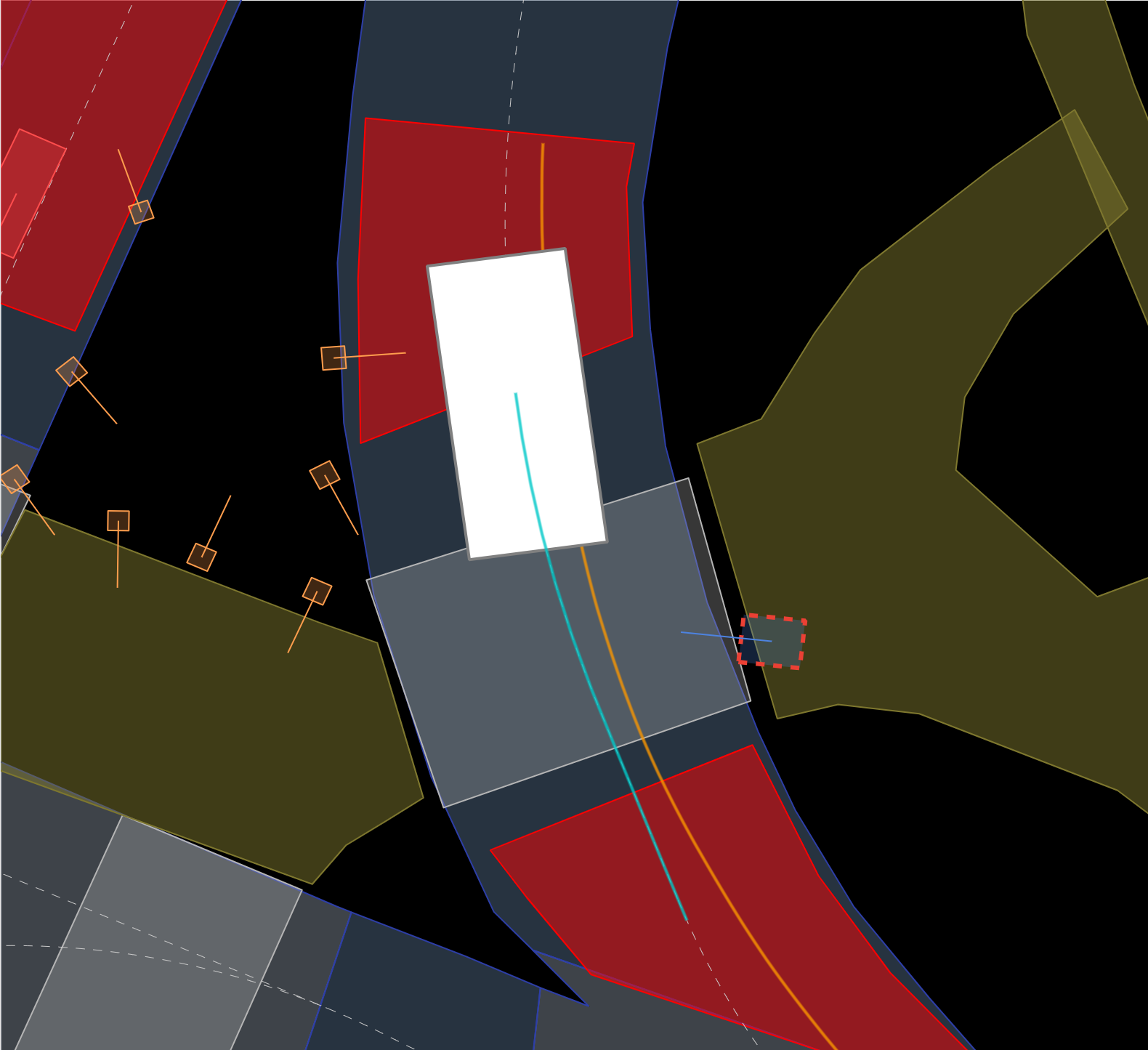}}
  \setcounter{subfigure}{0}
  \subfigure[Misdetection (Size)]{\label{fig:risk_d}\includegraphics[width=59mm]{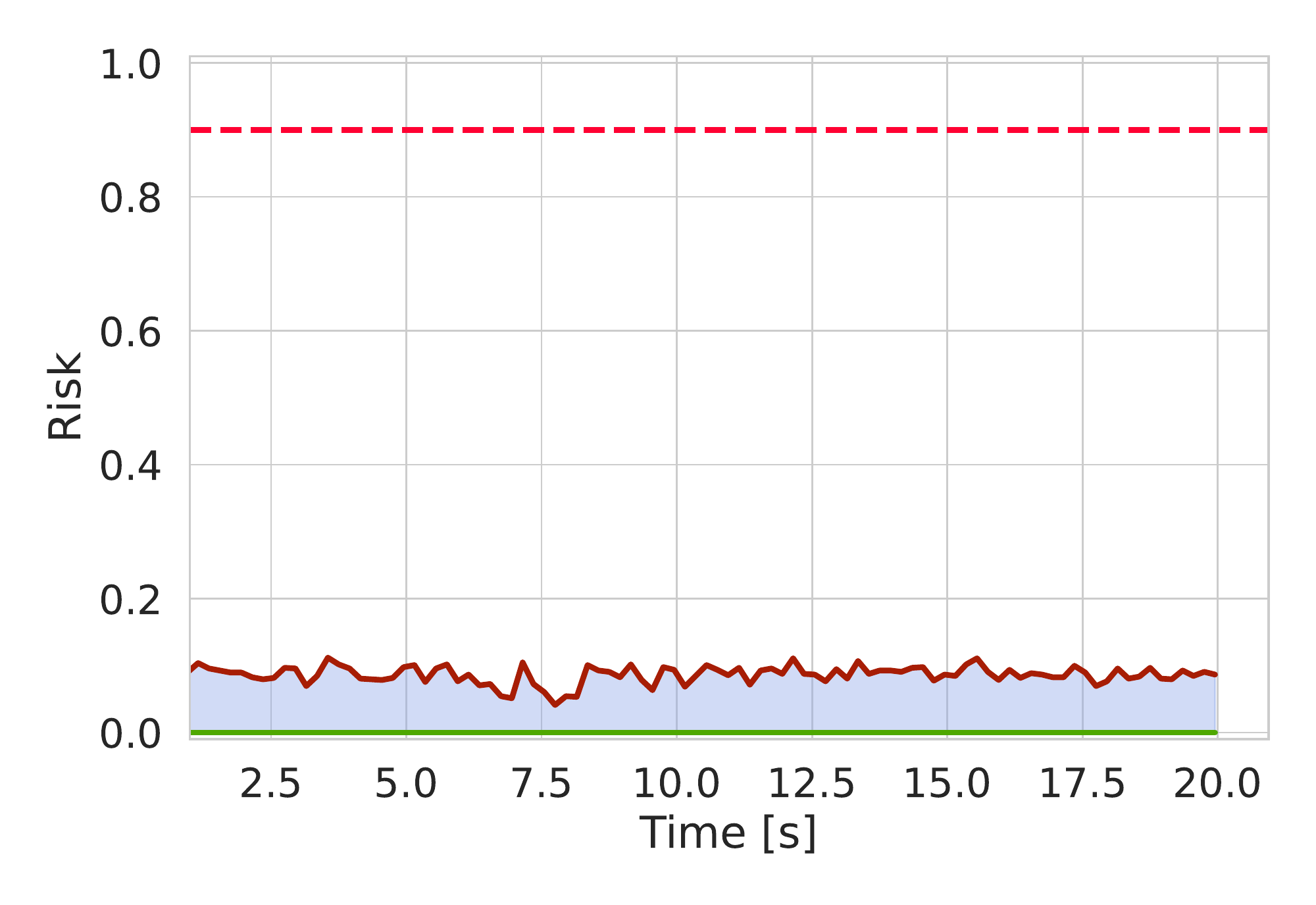}}
  \subfigure[Missing Obstacle]{\label{fig:risk_e}\includegraphics[width=59mm]{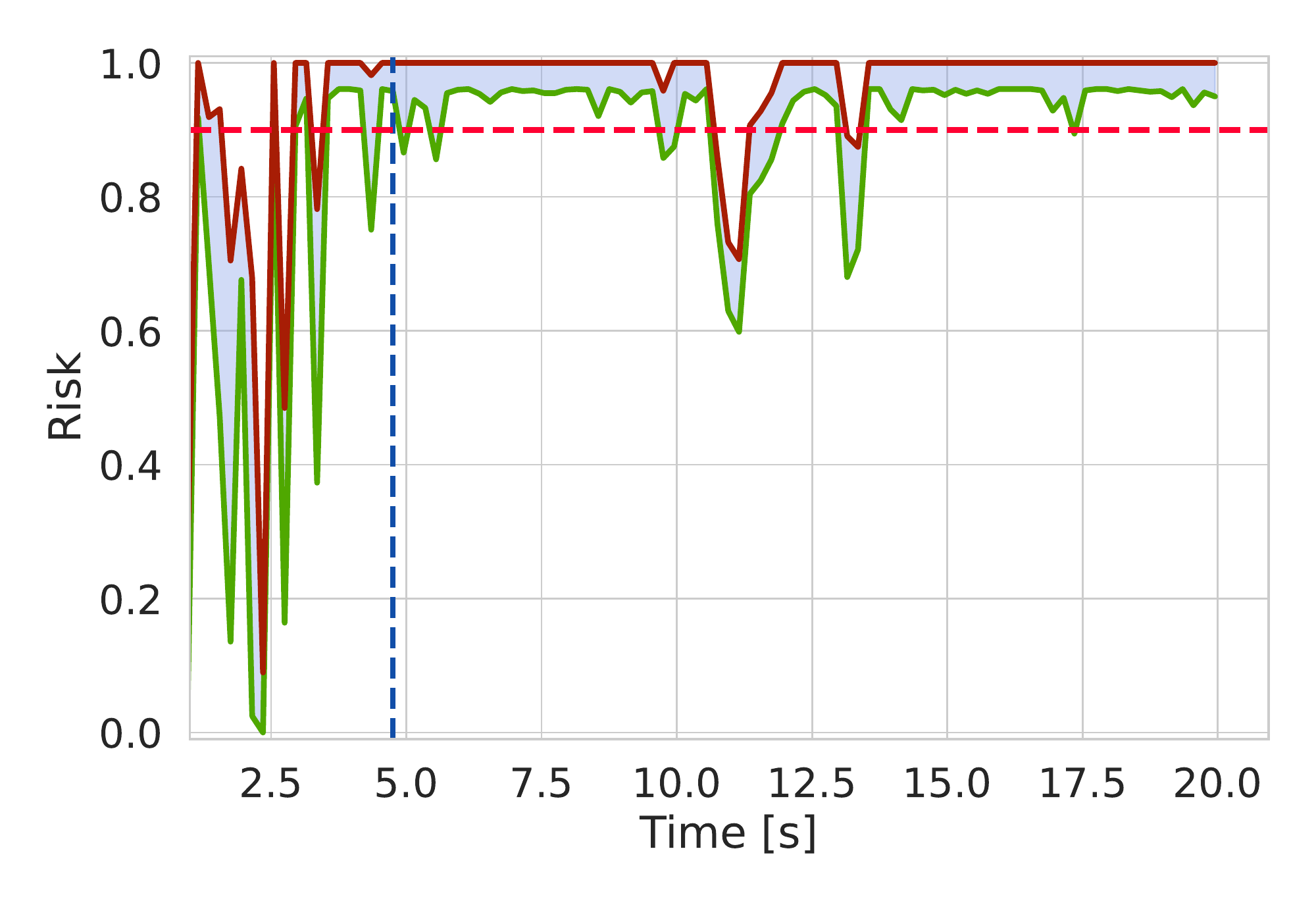}}
  \subfigure[Missing Pedestrian]{\label{fig:risk_f}\includegraphics[width=59mm]{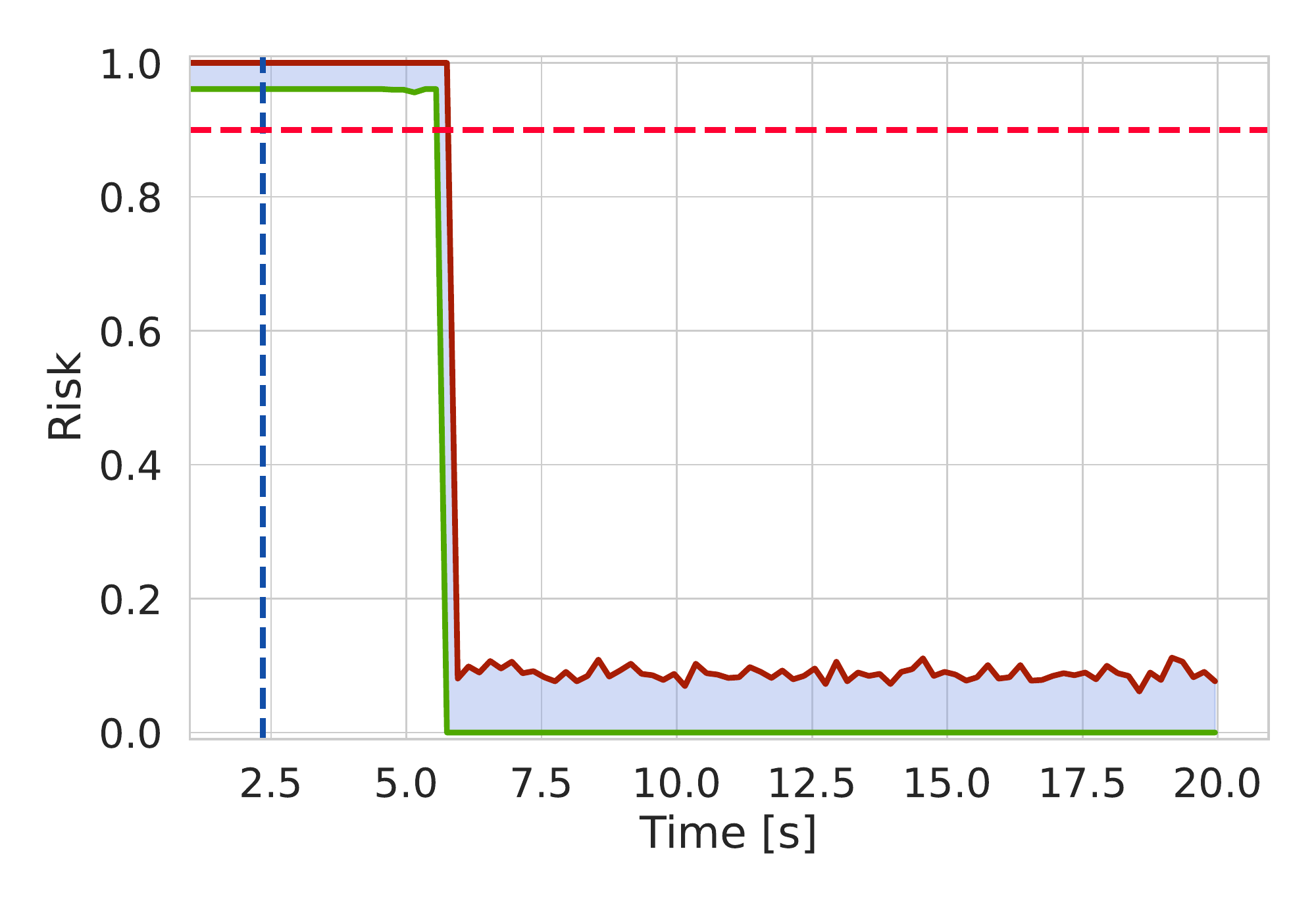}}
  \caption{
    {\bf Examples of scenarios and the associated estimated risk.}
    The top row shows the scenario, where the ego vehicle is represented as a white box, other vehicles as green boxes, and pedestrian as blue boxes.
    A dashed red line indicates the ground truth position and size of an agent, a solid line instead the one \perceived by the ego perception system.
    The bottom row shows the estimated risk for the corresponding scenario in the top row. 
    The horizontal dashed line represent the risk threshold $\riskth$.
    The red solid line represents the risk upper bound while the green line represent the risk lower bound.
    The vertical dashed blue line represents the time of the collision.
    It is worth noting that in our simulations, the behavior of the ego vehicle and the non-ego agents does not change after a collision, \ie the simulation continues running until the end of the scenario.
  }
  \label{fig:examples}
\end{figure*}

\begin{figure*}
  \centering
  \subfigure[Prediction Runtime]{\label{fig:timing_prediction}\includegraphics[width=55mm]{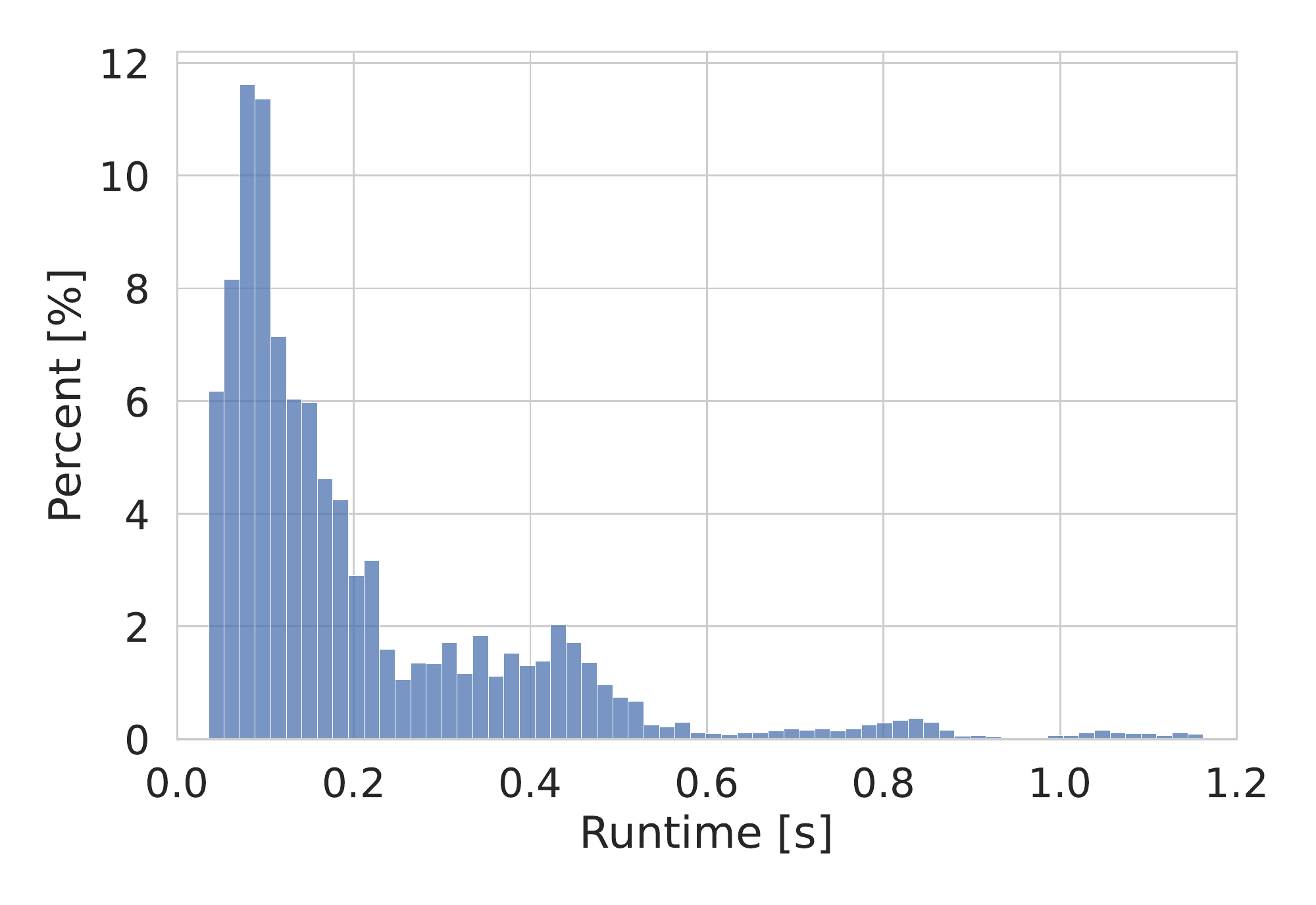}}
  \subfigure[Risk Estimation Runtime]{\label{fig:timing_risk}\includegraphics[width=55mm]{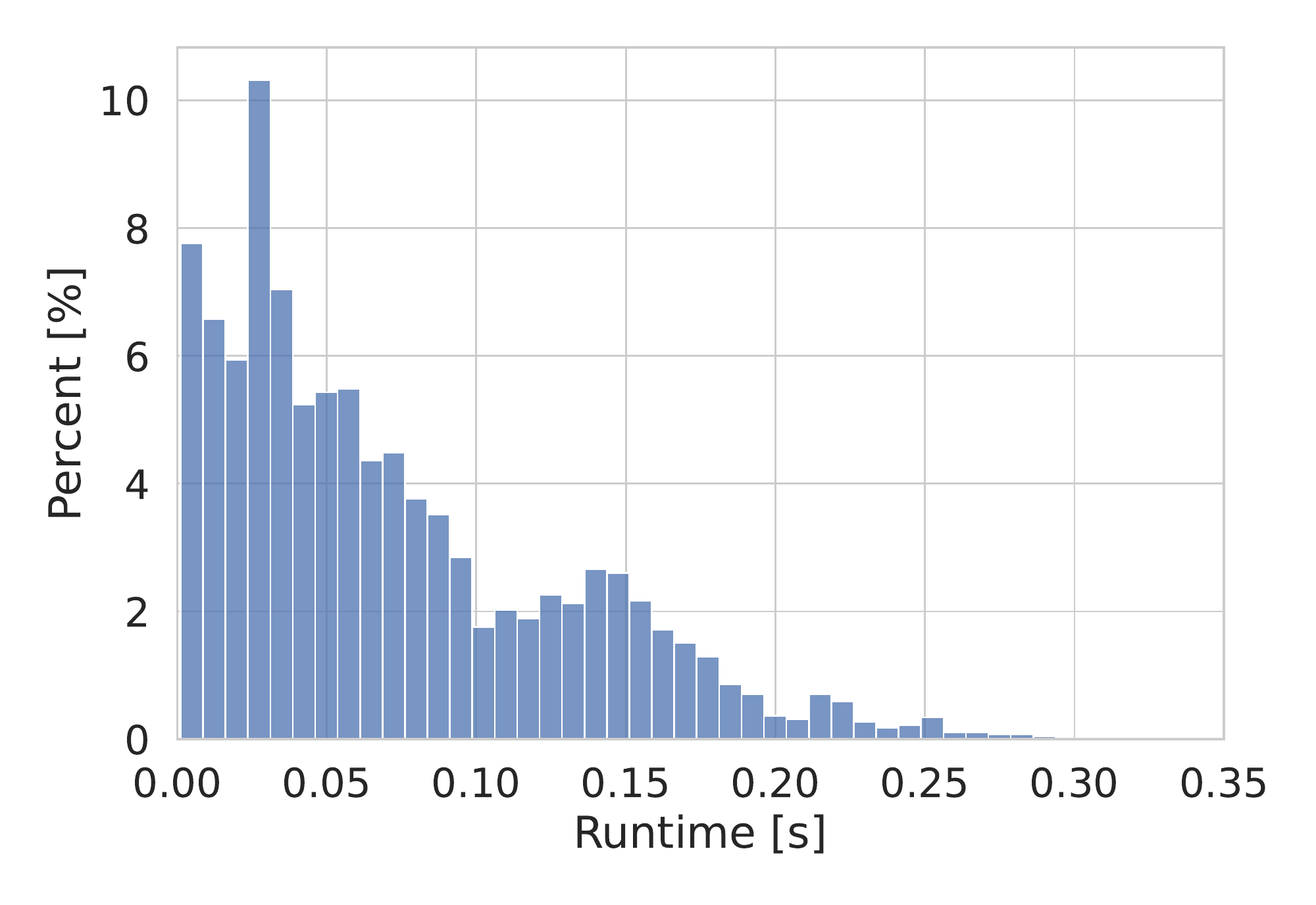}}
  \subfigure[Total Runtime]{\label{fig:timing_total}\includegraphics[width=55mm]{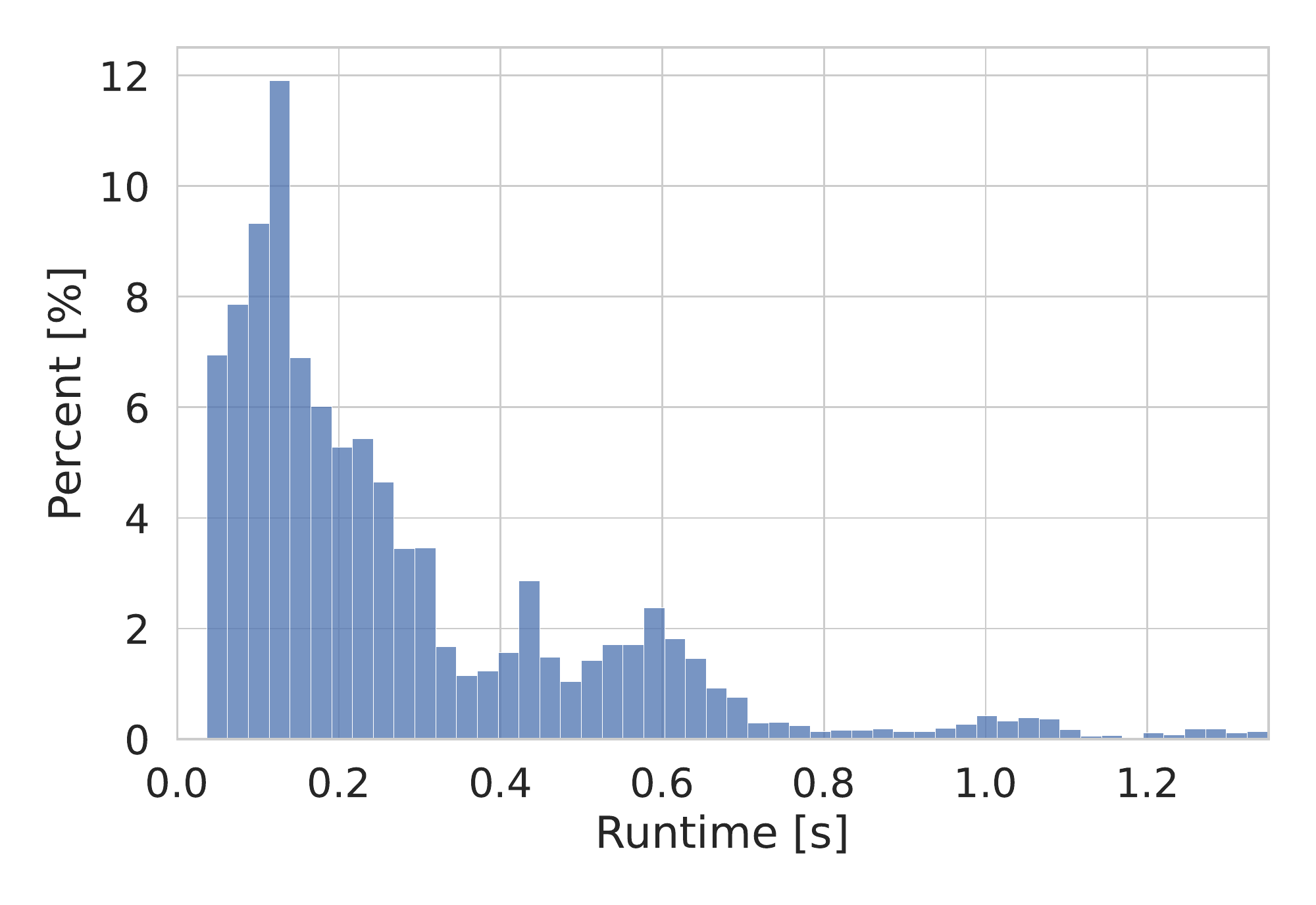}}
  \caption{
    {\bf Timing breakdown for the proposed approach using Momentum-Shaped Distance.}
    The prediction runtime averages at \SI{0.22}{\second} (median runtime: \SI{0.14}{\second}).
    The risk estimation runtime averages at \SI{0.07}{\second} (median runtime: \SI{0.06}{\second}).
    The total runtime averages at \SI{0.29}{\second} (median runtime: \SI{0.2}{\second}).
  }
  \label{fig:timing}
\end{figure*}


\section{Conclusions and Future Work} 
\label{sec:conclusion}

In this paper, we presented a framework to assess the risk a perception failure poses to the AV's motion plan. We achieved this by first identifying the perception failure mode, followed by synthesizing \plausible alternatives for the current scene, and then assessing how much more risk the AV faces in the \plausible scene as compared to the \perceived one. We formalized the notion of \taskaware risk as the 
\emph{$p$-quantile relative scenario risk} 
and then developed an algorithm to estimate it using i.i.d. samples. Additionally, we provide PAC bounds for our risk estimate which ensures the correctness of our algorithm with high probability. Finally, experimental evaluation of our approach revealed that our detector outperforms the baselines in terms of precision, recall, accuracy, and F1 score.

As part of our future work, we will develop the entire integrated task-aware perception monitor with each of the three building blocks (perception failure detection and identification module, \hypgen, task-aware risk estimator) and evaluate its closed-loop performance. 
We will also explore data-driven calibration and online adaptation of the risk aversion and the risk threshold parameters. 
Finally, we will work towards improving the runtime of our approach. As mentioned above, the computational bottleneck in our experiments arises from querying Trajectron++. We will use faster prediction networks, such as~\cite{kamenev22icra-predictionnet}, to speed up our computation.

\appendix

\subsection{Proof of \cref{thm:risk-bound}}\label{app:proofs}

Before proving~\cref{thm:risk-bound}, we introduce two useful results.
Let's start by noticing that the risk function in~\cref{eq:copula-based-risk} relies on the copula, and the CDFs of $A$ and $B$, namely $\Phi_A$ and $\Phi_B$.
Therefore, to provide the performance bound on the risk function $\riskfcn$, we first need to bound the copula value. To this end, we can use the well-known Fr\'echet–Hoeffding copula bounds.
\begin{theorem}[Fr\'echet–Hoeffding Theorem (\cite{nelsen07boo-introToCopulas}, Theorem 2.2.3)]\label{thm:Frechet_Hoeffding_copula}
  For any copula $\copula:[0,1]^d\to[0,1]$ and any $(u_1,\ldots,u_d)\in [0,1]^d$, the following bounds hold:
  \begin{equation}\label{eq:Frechet_Hoeffding_copula}
    W(u_1,\ldots,u_d) \leq \copula(u_1,\ldots,u_d) \leq M(u_1,\ldots,u_d),
  \end{equation}
  where
  \begin{equation*}
    \begin{aligned}
      W(u_1,\ldots,u_d) &= \max \left\{ 1-d+\sum_{i=1}^d u_i, 0 \right\}, \\
      M(u_1,\ldots,u_d) &= \min \left\{ u_1,\ldots, u_d\right\}.
    \end{aligned}
  \end{equation*}
\end{theorem}

In this paper, we are interested in the bi-dimensional copula, so we can apply \cref{thm:Frechet_Hoeffding_copula} to a copula $\copula(p,v)$ and obtain:
\begin{equation*}
  \max \left\{ p+v-1, 0 \right\} \leq \copula(p,v) \leq \min \left\{ p,v\right\}.
\end{equation*}
In particular, for the risk estimation, we take $v=\Phi_B\circ\Phi_A\inv(p)$.
However, we do not have access to the explicit expression of the two CDFs $\Phi_A$ and $\Phi_B$; therefore, we need to estimate them empirically.
The following theorem provides estimation bounds on the empirical CDFs.
\begin{theorem}[Dvoretzky–Kiefer–Wolfowitz Confidence Interval~\cite{dvoretzky56-empCDFInequality,massart90-empCDFInequality}]\label{thm:cdf_confidence_interval}
Let $\Phi$ be the CDF of an unknown distribution, and let $\Phi\at{n}$ the empirical CDF computed using $n$ i.i.d. samples from $\Phi$, then, with probability at least $1-\alpha$,
\begin{equation}\label{eq:cdf_confidence}
  \Phi\at{n}(x) - \epsilon(n,\alpha) \leq \Phi(x) \leq \Phi\at{n}(x) + \epsilon(n,\alpha),
\end{equation}
where  $\epsilon(n, \alpha) = \sqrt{\ln(2/\alpha)/{(2n)}}$.
\end{theorem}

We use~\cref{thm:cdf_confidence_interval} to estlablish bounds on $\Phi_B\circ\Phi_A\inv(p)$ in the next lemma.

\begin{lemma}\label{thm:cdf_composition_bound}
  Let $A$, $B$ be two random variables with CDFs $\Phi_A$ and $\Phi_B$, respectively.
  Let $\Phi_A\at{n}$ and $\Phi_B\at{n}$ be the empirical CDFs estimated using $n$ i.i.d. samples from $\Phi_A$ and $\Phi_B$, respectively.
  Then, with probability at least $1-\alpha$ we have:
  \begin{equation*}
    \lowbound{v}(p,\alpha,n) \leq \Phi_B\circ\Phi_A\inv(p) \leq \upbound{v}(p,\alpha,n),
  \end{equation*}
  where
  \begin{equation*}
    \begin{aligned}
      \upbound{v}(p,\alpha,n) &= \Phi_B\at{n}\circ\left[\Phi_A\at{n}-\epsilon(\alpha,n)\right]\inv(p) +\epsilon(\alpha,n)\\
      \lowbound{v}(p,\alpha,n) &= \Phi_B\at{n}\circ\left[\Phi_A\at{n}+\epsilon(\alpha,n)\right]\inv(p) - \epsilon(\alpha,n)
    \end{aligned}
  \end{equation*}
  and $\epsilon(\alpha,n) = \sqrt{\ln(2/\alpha)/{(2n)}}$.
\end{lemma}
\begin{proof}
We are interested in the quantity $\Phi_B\circ\Phi_A\inv(p)$.
Notice that $\Phi_A\at{n}$ and $\Phi_B\at{n}$ are increasing functions.
From \cref{thm:cdf_confidence_interval} we know that $\lowbound{x} \leq \Phi_A\inv(p) \leq \upbound{x}$ (see~\cref{fig:cdf_bounds}), where
\begin{equation}\label{eq:cdf_composition_bound_step1}
  \begin{aligned}
    \upbound{x}(p,\alpha,n) &= \left[\Phi_A\at{n}-\epsilon(\alpha,n)\right]\inv(p)\\
    \lowbound{x}(p,\alpha,n) &= \left[\Phi_A\at{n}+\epsilon(\alpha,n)\right]\inv(p)
  \end{aligned}
\end{equation} 
Similarly, we have $\lowbound{v} \leq \Phi_B\circ\Phi_A\inv(p) \leq \upbound{v}$ where
\begin{equation}\label{eq:cdf_composition_bound_step2}
  \begin{aligned}
  \upbound{v}(p,\alpha,n) &= \Phi_B\at{n}\circ\upbound{x}(p,\alpha,n) +\epsilon(\alpha,n)\\
  \lowbound{v}(p,\alpha,n) &= \Phi_B\at{n}\circ\lowbound{x}(p,\alpha,n) -\epsilon(\alpha,n)
  \end{aligned}
\end{equation}
Substituting \cref{eq:cdf_composition_bound_step1} into \cref{eq:cdf_composition_bound_step2}  we complete the proof
\end{proof}

\begin{figure}[htbp]
  \centering
  \includegraphics[width=0.45\textwidth]{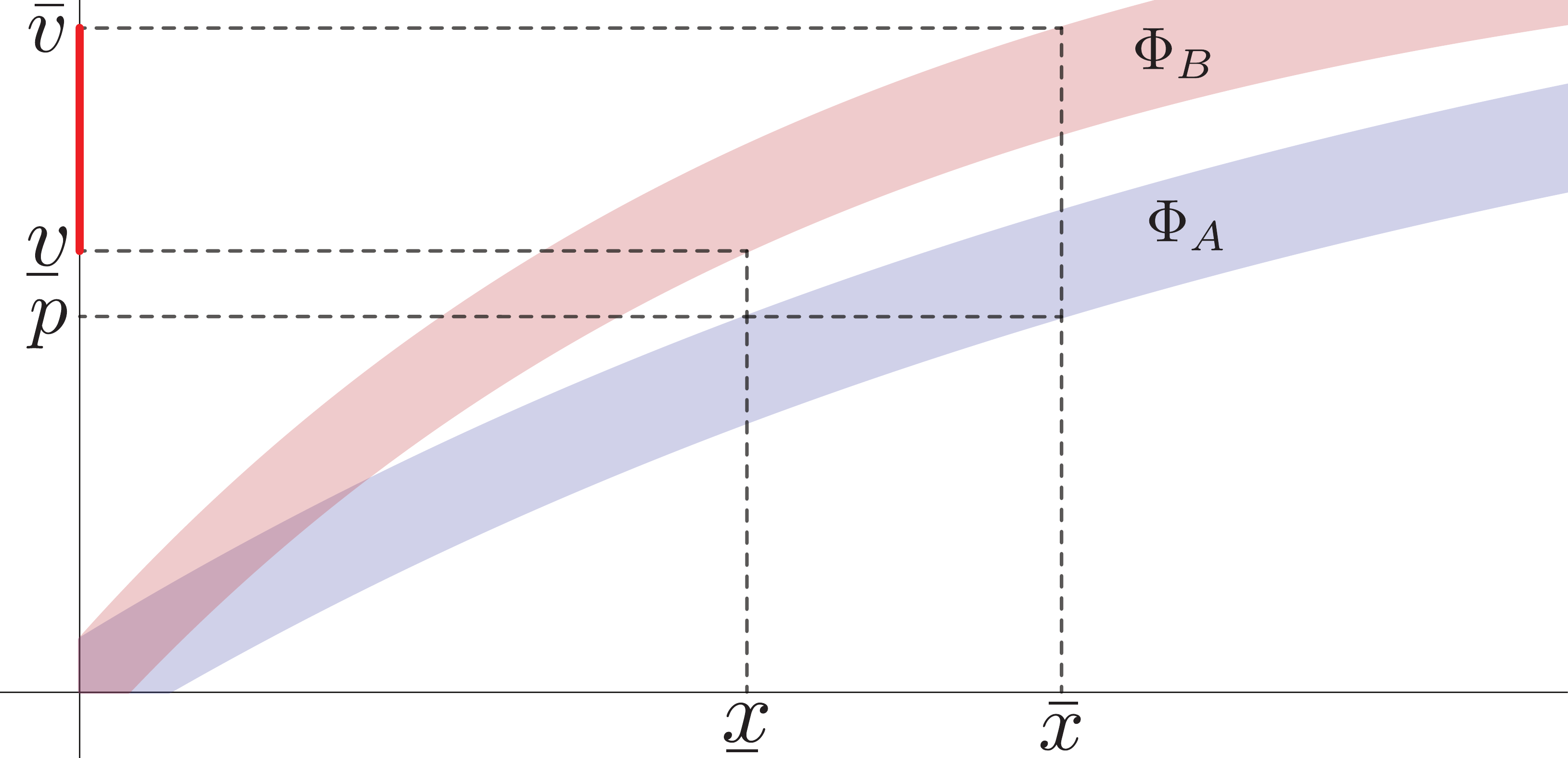}
  \caption{CDF composition bounds. 
  The shaded regions represent the the confidence intervals for the two CDFs, \ie $\Phi\at{n}(x)-\epsilon(n,\alpha)\leq\Phi(x)\leq\Phi\at{n}(x)+\epsilon(n,\alpha)$.
  The blue and orange regions represent $\Phi_A$ and $\Phi_B$ respectively.
  }
  \label{fig:cdf_bounds}
\end{figure}

We are now ready to prove \cref{thm:risk-bound}.

\begin{proof}[Proof of \cref{thm:risk-bound}]
    From \cref{thm:Frechet_Hoeffding_copula} we know that
    \begin{equation*}
      \max \left\{ p+v-1, 0 \right\} \leq \copula(p,v) \leq \min \left\{ p, v \right\},
    \end{equation*}
    where $v=\Phi_B\circ\Phi_A\inv(p)$. 
    We can use \cref{thm:cdf_composition_bound} to bound $v$, obtaining
    \begin{equation}\label{eq:copula-bound}
      \max(p+\lowbound{v}-1,0)\leq \copula(p, v) \leq \min(p,\upbound{v}).
    \end{equation}
    Since $\riskfcn(p)=1-\nicefrac{\copula(p,\Phi_B\circ\Phi_A\inv(p))}{p}$, we have $\copula(p,\Phi_B\circ\Phi_A\inv(p))=p\left(1-\riskfcn(p)\right)$.
    Substituting this into \cref{eq:copula-bound} completes the proof.
\end{proof}


\subsection{Cost Functions}\label{app:cost_functions}

Let $\statevar_e$, $\vv_e$ be the position and velocity of the ego vehicle.
Similarly, let $\statevar_a$, $\vv_a$ be the position and velocity of any non-ego agent.
Moreover, let $\nu$ be a term used to penalize the violation of traffic rules, such as driving on the wrong side of the road, or driving in the opposite direction of the traffic or crossing an intersection with a red traffic light.

The time-to-collision cost is defined as:
\begin{equation*}
  c_{\mathrm{TTC}} = 1 - \max_{a\in\mathrm{Agents}}\min\left\{\frac{\mathop{TTC}(x_e,x_a,v_e,v_a)}{m},1\right\} + \nu,
\end{equation*}
where $m$ represent a maximum value for the $\mathop{TTC}$ function, which outputs the time until a collision between the ego vehicle and another agent occurs (assuming constant velocity at their current heading direction), or is infinite if no collision occurs.
In our experiments we use $m=3$.
This cost function takes values in $[0,1]$, and higher values indicate smaller Time-To-Collisions, thus higher risk.

The Momentum Shape Distance is instead defined as:
{\small
\begin{align*}
  &c_{\mathrm{MSD}} = \max_{a\in\mathrm{Agents}} e^{\nicefrac{\epsilon\delta}{2}} + \nu, \\
  &\delta =  
    \left((\vxx_{a,\parallel}-\vxx_{e,\parallel}) (\vv_{a,\parallel} - \vv_{e,\parallel})\right)^2 +
    \left((\vxx_{a,\perp}-\vxx_{e,\perp}) (\vv_{a,\perp} - \vv_{e,\perp})\right)^2
\end{align*}}

\noindent
were we used the subscript $\parallel$ and $\perp$ to denote the projection of a vector along the parallel and perpendicular direction to the ego vehicle's heading, respectively.
The scaling factor $\epsilon$ weighs the importance of an agent: in our experiments we set $\epsilon = 0.5$ when the agent is a vehicle, and $\epsilon = 1$ when the agent is a pedestrian.

\bibliographystyle{plainnat}
\bibliography{references, spark}

\end{document}